\documentclass{article}
\usepackage[utf8]{inputenc}
\usepackage[margin=1in]{geometry}

\usepackage[T1]{fontenc}    
\usepackage{hyperref}       
\usepackage{url}            
\usepackage{booktabs}       
\usepackage{amsfonts}       
\usepackage{nicefrac}       
\usepackage{microtype}      
\usepackage{xcolor}         
\usepackage{colortbl}
\usepackage{wrapfig}
\usepackage{enumitem}
\usepackage{pythonhighlight}

\usepackage{microtype}
\usepackage{graphicx}
\usepackage{subcaption}
\usepackage{booktabs}
\usepackage{amsmath}
\usepackage{amssymb}
\usepackage{mathtools}
\usepackage{amsthm}
\usepackage{hyperref}
\usepackage{url}
\usepackage{booktabs}
\usepackage{amsfonts}
\usepackage{nicefrac}
\usepackage{microtype}
\usepackage{xcolor}
\usepackage{graphicx}
\usepackage{dsfont}
\usepackage{thmtools, thm-restate}
\usepackage[numbers]{natbib}
\usepackage{algorithmic}
\usepackage{algorithm}
\usepackage{pythonhighlight}

\usepackage{multirow}
\usepackage{xspace}

\usepackage{enumitem}
\usepackage{fontawesome}

\newcommand{\blockcomment}[1]{}
\renewcommand{\l}{\left}
\renewcommand{\r}{\right}

\DeclareMathOperator*{\argmin}{argmin}
\DeclareMathOperator*{\argmax}{argmax}

\newcommand{\nn}{\nonumber}

\newcommand{\T}{\top}

\newcommand{\E}{\mathbb{E}}
\newcommand{\R}{\mathbb{R}}
\renewcommand{\P}{\mathbb{P}}
\newcommand{\I}{\mathds{1}}

\newcommand{\cB}{\mathcal{B}}

\newcommand{\cL}{\mathcal{L}}

\newcommand{\cR}{\mathcal{R}}

\renewcommand{\a}{\alpha}
\renewcommand{\b}{\beta}
\renewcommand{\d}{\delta}

\newcommand{\e}{\varepsilon}
\newcommand{\g}{\gamma}
\newcommand{\n}{\eta}
\newcommand{\s}{\sigma}

\newcommand{\pI}{p_{\mathrm{I}}}
\newcommand{\pII}{p_{\mathrm{II}}}

\newcommand{\hfjn}{\hat{G}^{(j)}_n}

\newcommand{\uq}{\underline{q}}
\newcommand{\oq}{\overline{q}}
\newcommand{\pt}{\frac{p}{2}}

\theoremstyle{plain}
\newtheorem{thm}{Theorem}[section]
\newtheorem{prop}[thm]{Proposition}
\newtheorem{lem}[thm]{Lemma}

\theoremstyle{definition}
\newtheorem{defn}[thm]{Definition}

\theoremstyle{remark}

\title{Subgroup Discovery with the Cox Model}

\author{Zachary Izzo and Iain Melvin\thanks{NEC Labs America. Correspondence to \texttt{zach@nec-labs.com}.}}
\date{}

\begin{document}

\maketitle

\begin{abstract}
We study the problem of subgroup discovery for survival analysis, where the goal is to find an interpretable subset of the data on which a Cox model is highly accurate.
Our work is the first to study this particular subgroup problem, for which we make several contributions.

Subgroup discovery methods generally require a ``quality function'' in order to sift through and select the most advantageous subgroups. We first examine why existing natural choices for quality functions are insufficient to solve the subgroup discovery problem for the Cox model.
To address the shortcomings of existing metrics, we introduce two technical innovations: the \emph{expected prediction entropy (EPE)}, a novel metric for evaluating survival models which predict a hazard function; and the \emph{conditional rank statistics (CRS)}, a statistical object which quantifies the deviation of an individual point to the distribution of survival times in an existing subgroup. We study the EPE and CRS theoretically and show that they can solve many of the problems with existing metrics.

Having established the fundamentals of the problem, we then turn to methodology for solving it. To this end, we introduce a total of eight algorithms for the Cox subgroup discovery problem. The main algorithm, which is based on the DDGroup framework of \cite{izzo23subgroup}, is able to take advantage of both the EPE and the CRS, allowing us to give theoretical correctness results for this algorithm in a well-specified setting.
We evaluate all of the proposed methods empirically on both synthetic and real data. The experiments confirm our theory, showing that our contributions allow for the recovery of a ground-truth subgroup in well-specified cases, as well as leading to better model fit compared to naively fitting the Cox model to the whole dataset in practical settings. Lastly, to showcase the utility of the subgroups themselves beyond improving model predictions, we conduct a case study on jet engine simulation data from NASA. The discovered subgroups uncover known nonlinearities/homogeneity in the data, and which suggest design choices which have been mirrored in practice. \href{https://github.com/zleizzo/cox-subgroup}{\faGithub}
\end{abstract}

\section{Introduction}
The Cox model \citep{cox1972regression, cox1975partial} is widely used in fields ranging from biostatistics to manufacturing \citep{hosmer2008textbook, kalbfleisch2011textbook}, both for prediction \citep{huo2024ttepretraining} and for qualitative inference about the data \citep{liu2022systematic, liu2024characterizing}.
While the Cox model is appealing for its ease of interpretation, it makes restrictive modeling assumptions which are known not to hold in some practical scenarios \citep{hernan2010hazards}. Other methods for survival analysis, including many deep learning-based approaches, have been proposed by the ML community \citep{wang2019mlsurv}. These models are more expressive and therefore do not require making restrictive assumptions about the data, but this comes at the cost of interpretability. Especially in high-stakes settings, or when we want to use the model not only for prediction but for qualitative inference about the data, this may be unacceptable \citep{rudin2019stop}. Thus, it is valuable to develop methods which improve the modeling capacity of the Cox model without sacrificing interpretability.

To fill this gap, we use a subgroup discovery-based approach \citep{lipkovich2023modern}. The idea is that while the Cox model may not be a good fit for the entire dataset, there may be local subpopulations of the data (referred to as \emph{subgroups}) where the Cox model is appropriate. In order to preserve the overall interpretability of the method, it is desirable that the subgroup descriptions also be interpretable. We follow prior work \citep{izzo23subgroup}
and define the subgroups by thresholding different feature values. The resulting subgroups correspond to axis-aligned boxes in feature space. The interpretable subgroup definitions mean that the subgroups themselves may be useful in defining meaningful subpopulations with homogeneous, predictable outcomes, which may be useful for follow-up studies. In short, rather than sacrificing accuracy or interpretability, we sacrifice our ability to model the entire population simultaneously, and our goal can be summarized as:

\begin{center}
    \boxed{\textrm{Find interpretable subsets of the data in which a Cox model is an excellent fit.}}
\end{center}

In this paper, we address several fundamental questions for the problem of subgroup discovery with the Cox model.
We first look at the barriers to solving the problem using existing knowledge. Specifically, subgroup discovery methods generally rely on a ``quality function'' to evaluate potential subgroups, allowing these algorithms to sift through and select a favorable subgroup. In the context of the Cox model, two natural choices for quality functions are Harrell's C-index \citep{harrell1982cindex} and the Cox partial likelihood \citep{cox1975partial}. By means of counterexamples and empirical evaluations, we show that both of these metrics have undesirable properties which make them unsuitable for solving the Cox subgroup discovery problem.

To address the shortcomings of existing metrics, we make two novel technical contributions: the \emph{expected prediction entropy (EPE)}, an alternative metric for the performance of a Cox model (or any survival model which estimates a hazard function); and the \emph{conditional rank statistics (CRS)}, which measures the plausibility that a test point follows the same Cox model as some reference group of points. We study the theoretical properties of both of these quantities and show that they address some of the shortcomings of the existing metrics considered earlier.

Having established fundamental results on the Cox subgroup discovery problem, we next develop methodology for solving it. Indeed, the EPE and CRS, in addition to being motivated by the shortcomings of the C-index and partial likelihood, are also motivated by the general subgroup discovery framework of \cite{izzo23subgroup}. The idea is to first select a small ``core group'' of points with a good fit to the Cox model; this stage is accomplished with the EPE. Next, we examine the remaining points in the dataset and ``reject'' those which could not feasibly follow the same model as the core group; in the case of survival analysis, this can be accomplished with the CRS. Finally, we define the subgroup by expanding the core group as much as possible without including any rejected points. By taking advantage of both the EPE and CRS, we prove that this ``survival analog'' of DDGroup can provably recover a ground truth subgroup in a well-specified setting.

To the best of our knowledge, we are the first to consider this formulation of the subgroup discovery problem with the Cox model. As a result, there are no existing baseline methods, so in addition to our primary algorithm, we also introduce seven other methods for comparison. We then evaluate all eight algorithms (DDGroup + seven baselines) against the non-subgroup discovery approach of simply fitting a Cox model to the entire dataset on both real and synthetic data. These experiments confirm our theory and show the practical utility of the proposed methods. In addition to evaluating prediction performance, we also conduct a case study on simulated jet engine failure data from NASA \citep{saxena2008damage} to show the value of the subgroups themselves. We are able to recover subgroups corresponding to known nonlinearities in the data and which suggest follow-up studies and design choices which have indeed been mirrored in practice, showing the value of these subgroup methods for hypothesis generation.

\paragraph{Summary of Contributions}
\begin{enumerate}
    \item We analyze the shortcomings of existing metrics, such as the C-index and partial likelihood, for solving the problem of subgroup discovery with the Cox model.
    \item We introduce the \emph{expected prediction entropy (EPE)} as an evaluation metric for survival model accuracy.
    We derive several properties of the EPE to assist in its interpretation as an evaluation metric.
    \item We introduce the \emph{conditional rank statistics (CRS)}, which quantifies the deviation of an individual point to the distribution of survival times in an existing subgroup.
    \item We introduce eight algorithms for subgroup discovery with the Cox model. For one of the algorithms, DDGroup, which makes use of both the EPE and the CRS, we are able to show theoretical correctness guarantees.
    \item We apply these algorithms to both real and synthetic datasets and analyze the subgroups discovered by each method. We also conduct a case study to demonstrate the utility of discovered subgroups for hypothesis generation as well as improved predictive power.
\end{enumerate}

\section{Background on Survival Analysis and the Cox Model} \label{appendix: background}
A \emph{survival time} is a non-negative random variable $T$ which describes the amount of time until an event of interest. Examples of commonly modeled events include the onset of a disease, the death of a patient, the time at which a customer stops using a product or platform, or the failure of a mechanical component. The arbitrary event to be modeled is referred to as a \emph{failure}. Unlike more typical regression tasks in machine learning where the goal is to give a point estimate of a continuous-valued target, the goal of survival analysis is usually to model the \emph{distribution} of $T$ conditional on some associated covariates $X \in \R^d$.

Natural modeling targets for describing the distribution of $T$ include standard probabilistic quantities such as the probability density function (pdf) or cumulative distribution function (cdf) of $T$, conditional on the features $X$, and indeed some survival analysis methods take this approach. A more common target, however, is the \emph{hazard function}, defined as
\begin{equation} \label{eq: cts hazard}
    \lambda(t; x) = \lim_{dt \rightarrow 0} \frac{\P(t \leq T \leq t+dt \: | \: T \geq t, X=x)}{dt}.
\end{equation}
The hazard function can be thought of as an instantaneous rate of failure in the infinitesimal time interval $[t, t+dt)$, conditional on surviving up to time $t$ and on the features $X=x$. The hazard function is related to more standard quantities like the pdf or cdf. Specifically, letting $F(t, x) = \P(T \leq t \: | \: X = x)$ be the cdf and $f(t, x)$ the associated pdf (assuming one exists), we have the following identities:
\[S(t, x) := 1 - F(t, x) = \exp\l\{-\int_0^t \lambda(u, x)\, du \r\},
\quad \quad f(t, x) = \lambda(t, x) S(t, x).\]
The complement $S(t,x)$ of the cdf is referred to as the \emph{survival function}. The existence of these formulas shows that determining the hazard function completely specifies the distribution of $T|X$, as it completely specifies the pdf or cdf. In a biomedical context, the hazard function has several advantageous properties which make it a natural modeling target, including but not limited to interpretability. For instance, a patient in remission from cancer would naturally be more interested in knowing the conditional probability of a recurrence given that they have not experienced one yet, rather than an absolute probability which is more easily described by the cdf \cite{tian2015survnotes}.

The Cox model posits a particular semiparametric form for the hazard function which implies that a unit change in each covariate has a multiplicative effect on the hazard function, i.e.,
\begin{equation} \label{eq: cox}
\lambda(t; x) = \lambda_0(t)\exp(\b^\T x)
\end{equation}
for some coefficients $\beta$. Note that the higher the value of the log relative hazard $\b^\T x$, the faster the model predicts the unit with features $x$ will fail. In particular, this means that if $\b^\T x > \b^\T x'$, then the model predicts that the survival time for $x$ will be less than that of $x'$ ($x$ will fail first).

Let $L_i$ denote the index of the individual who fails at time $t_i$. Let $T_\ell$ denote the random failure time for the $\ell$-th individual, and define $R(t)$ to be the \emph{risk set} at time $t$, i.e. the set of individuals $R(t) = \{ \ell \: : \: T_\ell \geq t\}$ who have not failed before time $t$. It can be shown that
\begin{equation} \label{eq: partial likelihood term}
    \P(L_i = \ell \: | \: T_{L_i} = t_i, R(t_i) = R_i) = \frac{\exp(x_\ell^\T \b)}{\sum_{j\in R_i} \exp(x_j^\T \b)}.
\end{equation}
Given the failure times $t_1 \leq \ldots \leq t_n$ and associated features $x_i$ and failure indicators $\d_i$, \cite{cox1972regression} then proposed estimating $\b$ by maximizing the log partial likelihood
\begin{equation} \label{eq: log partial likelihood}
    \cL(\b) := \sum_{1\leq i \leq n, \: \d_i=1} \l[ x_i^\T \b -  \log\l( \sum_{j \geq i} \exp(x_j^\T \b) \r) \r].
\end{equation}
While each term in the partial likelihood is a likelihood in the traditional sense, \cite{cox1975partial} showed that $\exp(\cL(\b))$ is \emph{not} a marginal or conditional likelihood (unless one makes restrictive assumptions on the censoring patterns/failure times). Nevertheless, maximizing \eqref{eq: log partial likelihood} still enjoys many of the same properties as traditional MLE, such as asymptotic normality and consistency \citep{cox1975partial}.

\section{Expected Prediction Entropy} \label{sec: epe}
As discussed in the introduction, we are guided in part by the desire to adapt the DDGroup framework of \cite{izzo23subgroup} to the survival analysis setting. The first step of this framework is to select a small ``core group'' of points where the Cox model gives a good fit. In general, we also want precise metrics to evaluate the quality of discovered subgroups, especially in realistic settings where ground truth subgroup descriptions may not be known. In this section, we discuss the shortcomings of existing metrics which appear to be natural choices for these tasks, and introduce the expected prediction entropy (EPE) to address these shortcomings.

\subsection{Inadequacy of Existing Metrics}
One of the most common measures of model accuracy in survival analysis is Harrell's C-index \citep{harrell1982cindex}, which measures the fraction of comparable units for which the earlier failure time coincides with the model's prediction. This metric is similar to the 0-1 loss in classification, and while it is a useful summary statistic, it is insensitive to the confidence of the model's prediction: a model prediction which was incorrect with 51\% confidence is penalized equally to a prediction which was incorrect with 99\% confidence.
As a result, evaluation using only the C-index can obscure the existence of small subpopulations with qualitatively different behavior from the rest of the population. In the context of subgroup discovery, these subpopulations are exactly what we want to detect, making the C-index inadequate for this task.

A natural alternative to consider is the partial likelihood.
While the partial likelihood does take model confidence into account, it is not suitable for \emph{comparing} different groups of data. This is because the value of the partial likelihood depends on the size of the risk sets $R_i$. For instance, if the first unit to fail out of 1000 units was given a predicted $10\%$ chance of being the first to fail by the model, this could be considered a very confident and accurate prediction (a $100\times$ improvement over a random guess, which would assign each unit a $1/1000$ chance of failure). On the other hand, if only two units were at risk and the model assigned a $10\%$ chance of failure to the unit which failed first, this would constitute a confident but inaccurate prediction. However, these two scenarios  contribute equally to the partial likelihood.

\subsection{EPE Definition}
Let $\lambda(t, x)$ be the true hazard function. Conditional on a failure occurring at time $t$ among two units with features $x_1$ and $x_2$, the probability that $x_1$ experiences failure is 
\begin{equation} \label{eq: failure prob}
\lambda(t, x_1)/(\lambda(t, x_1) + \lambda(t, x_2)).
\end{equation}
Given a survival model which predicts an instantaneous hazard rate $\hat{\lambda}(t, x)$, we can evaluate our model by measuring its ability to discriminate between which of two units at risk will fail.
\begin{defn}[Expected Prediction Entropy]
Let $P$ be a probability distributions over $\R^d \times \R_{\geq 0}$ which denotes the joint distribution of a (feature, survival time) pair. Let $(X, T), \, (X', T') \sim P$ be two i.i.d. draws from $P$, let $T^* = \min\{T, T'\}$, and define $Y = \I\{T \leq T'\}$. Let $\hat{\lambda}$ be an estimate for the hazard function which defines the distribution of $T$ conditional on $X$, and let $R \subseteq \R^d$ be a sub-region of the feature space. We define the \emph{expected prediction entropy (EPE)} as $\mathrm{EPE}(\hat{\lambda}, R) = $
\begin{equation} \label{eq: survival ce}
\E\l[-Y\log\frac{\hat{\lambda}(T^*, X)}{\hat{\lambda}(T^*, X) + \hat{\lambda}(T^*, X')} - (1-Y)\log\frac{\hat{\lambda}(T^*, X')}{\hat{\lambda}(T^*, X) + \hat{\lambda}(T^*, X')} \: \Bigg | \: X, X' \in R\r].
\end{equation}
\end{defn}

\paragraph{Specialization to the Cox Model}
The EPE has a particularly interesting interpretation when $\hat{\lambda}$ is given by a Cox model, i.e., $\hat{\lambda}(t, x) = \lambda_0(t)e^{\b^\T x}$. In this case, \eqref{eq: survival ce} reduces to
\begin{equation} \label{eq: cox epe}
\E\l[-Y\log\frac{1}{1+e^{-\b^\T (X - X')}} - (1-Y)\log\frac{1}{1+e^{\b^\T (X - X')}} \: \bigg| \: X, X' \in R\r].
\end{equation}
Observe that this is the standard cross entropy loss for a logistic model trained to predict the label $Y$ from the feature differences $X - X'$.
We remark that the expression \eqref{eq: cox epe} appeared in \cite{steck2007ranking} as a lower bound for the C-index. The authors use this lower bound directly to train a Cox model, instead of the standard partial likelihood.
\cite{kvamme2019time} used the same expression as an approximation to the partial likelihood, using a risk set of size 1 to avoid memory constraints during model training.
\cite{vauvelle2024differentiable} also explored this expression in the context of ranking losses, which are again used to train relative risk models.
To the best of our knowledge, we are the first to explore the usefulness and properties of the EPE as an \emph{evaluation metric}, not merely as a loss function.

\paragraph{Estimating EPE Empirically}
Let $\{(x_i, t_i, \d_i)\}_{i=1}^n \subseteq \R^d \times \R_{\geq 0} \times \{0, 1\}$ be a survival dataset with features $x_i$, event times $t_i$, and censoring indicators $\d_i$. An empirical estimate of the EPE is given by
\begin{equation} \label{eq: empirical epe}
-\frac{1}{N} \sum_{i \: : \: \d_i = 1} \sum_{j \in R_i} \log\frac{\hat{\lambda}(t_i, x_i)}{\hat{\lambda}(t_i, x_i) + \hat{\lambda}(t_i, x_j)},
\end{equation}
where $R_i = \{j \: : \: t_j > t_i\}$ is the risk set at time $t_i$ (minus the $i$-th datapoint itself) and $N = \sum_{i \: : \: \d_i = 1} |R_i|$ is the total number of comparable event times. In the case that there is no censoring (i.e., $\d_i=1$ for all $i$), \eqref{eq: empirical epe} gives an unbiased estimate for \eqref{eq: survival ce}. In the presence of censoring, the fact that we can only compare two datapoints when the first event time was uncensored may introduce a bias.

\subsection{Properties of the EPE}
We first show that when the data are generated by a Cox model, the EPE is a proper scoring rule in the sense that it is minimized if and only if the estimated Cox coefficients match the ground truth. By equation \eqref{eq: failure prob}, conditional on $T^*$, $X$, and $X'$, $Y$ is a Bernoulli random variable with parameter $p = \lambda(T^*, X)/(\lambda(T^*, X) + \lambda(T^*, X'))$. Since the cross entropy loss is a proper scoring rule, the minimum of \eqref{eq: survival ce} occurs when ratio of the estimated hazard functions equals its true value, i.e., when
\[ \frac{\hat{\lambda}(t, x)}{\hat{\lambda}(t, x) + \hat{\lambda}(t, x')} = \frac{\lambda(t, x)}{\lambda(t, x) + \lambda(t, x')} \]
for all $t, x, x'$ (potentially except for a set of measure 0). In particular, in the case of the Cox model, this implies that the Cox coefficients $\b$ must be correct. This is the content of Proposition~\ref{thm: proper scoring rule}.
\begin{restatable}{prop}{proper} \label{thm: proper scoring rule}
Suppose that the ground truth hazard function follows the Cox model, i.e., $\lambda(t, x) = \lambda_0(t)e^{\b^\T x}$. Then the EPE is minimized iff $\hat{\b}^\T (X-X') = \b^\T (X-X')$ with probability 1.
\end{restatable}
We remark that in the more general setting, as discussed in the derivation, the EPE is also minimized by any scalar multiple of the ground truth (full) hazard function. This is unavoidable, since the EPE is relative in nature and scaling the ground truth hazard will not change the relative probability of failure of one unit over another. However, multiplication of the hazard by a scalar also does not change the Cox coefficients, as this multiplication is absorbed into $\lambda_0(t)$, so the minimum EPE will uniquely identify the correct Cox coefficients. As the semiparametric nature of the Cox model means it is also inherently relative, this is the best we can hope for in terms of a proper scoring rule.

Next, we show that as long as the Cox model is well-specified, the EPE implicitly favors larger groups. This is beneficial in practical applications as we would like our conclusions to be applicable to as much of the data as possible.
\begin{restatable}{thm}{decepe} \label{thm: decreasing epe}
Let the joint data distribution $P$ be such that the marginal distribution of the features $P|_X$ is uniform on a region $\mathcal{B} \subseteq \R^d$. Let 
$R=\prod_{i=1}^d [a_i, b_i]$, $R'=\prod_{i=1}^d [a'_i, b'_i]$
be axis-aligned boxes such that $R, R'\subseteq \mathcal{B}$ and such that $|a_i-b_i| \leq |a'_i - b'_i|$ for all $i$. Further suppose that $T|X$ follows a Cox model with coefficients $\b$ whenever $X \in R \cup R'$. Then
$\mathrm{EPE}(\b, R') \leq \mathrm{EPE}(\b, R)$.
\end{restatable}
Because the minimum EPE depends on the intrinsic difficulty of distinguishing between units, we make the assumption that the features are uniformly distributed to normalize away this potential source of variation. For an easy counterexample of why this theorem may not hold without this assumption, consider features which are drawn uniformly at random from $\cB$ $1\%$, but with probability $99\%$ they are equal to some fixed $x_0\in\cB$. As long as the region $R$ does not contain $x_0$, we obtain the same behavior as in the theorem, as the features will be conditionally uniformly distributed in $R$. However, once $R$ expands to include $x_0$, with high probability we will have $X=X'$ and the failure probability is 50-50, leading to a larger EPE. This counterexample can also be smoothly approximated even with purely continuous distributions.

In general, the practical interpretation of Theorem~\ref{thm: decreasing epe} is that provided that the Cox model is always well specified, the EPE favors a group where the intrinsic difficulty of distinguishing between units is the lowest; in the case of uniformly distributed features, larger groups mean that units will tend to be farther apart, making them easier to classify.

\section{Conditional Rank Statistics} \label{sec: crs}
The EPE provides an improved metric for the fit of the Cox model to a group of points, and can be used for the core group selection step in the DDGroup framework. Next, we turn to the second phase of DDGroup, where we must ``reject'' points which cannot feasibly follow the same model as the core group. For this task, we introduce the conditional rank statistics (CRS).

\subsection{Motivation and Counterexamples}
\label{sec: epe counterex}
In the discussion of Theorem~\ref{thm: decreasing epe}, we noted that the result implies that the EPE favors regions where units are intrinsically easier for the Cox model to distinguish, and this can be impacted not only by the hazard function but also the feature distribution.
In particular, the EPE may prefer a region $R_1$ which does not follow a Cox model to a region $R_2$ which does follow a Cox model, provided that the units in $R_1$ are somehow ``intrinsically easier'' to distinguish between than the units in $R_2$. For a precise example, consider the following data generation setup: the dataset consists of 1D features in $[0, 1]$, and the ground truth hazard function is
\[
\lambda(t, x) = \begin{cases}
e^{m x} & 0 \leq x < c \\
e^{m x - b} & c \leq x \leq 1
\end{cases}.
\]
The same Cox model works in the left and right subintervals, but not with the same baseline hazard; thus, the Cox model is not well-specified for the entire interval using only the feature $x$.
Figure~\ref{fig: epe counterex} shows a heat map of the numerical value of the EPE for different regions $R$ with $m=10$, $b=2$, and $c=0.4$. We use a Monte Carlo estimate for the EPE with $n=4000$ total points. The region $[a, b]$ corresponds to the box whose bottom-left corder is at coordinate $(a, b)$ in the graph with the heat map value determined by the EPE for this region. Regions below the horizontal dashed red line are subsets of the left subinterval $[0, 0.4]$ (and therefore a Cox model holds in these regions); regions to the right of the vertical dashed red line are subsets of the right subinterval $[0.4, 1]$ (and therefore a Cox model holds in these regions as well, though not with the same baseline hazard function as the left subinterval). Regions in the upper-left quadrant (above the horizontal dashed red line and to the left of the vertical dashed red line) contain points in both the left and right subintervals, and therefore the Cox model does not hold. 
The largest region in which the Cox model is well-specified and with the lowest EPE is $[0.4, 1]$, indicated by the cyan square. However, the minimum EPE is actually obtained by the whole interval $[0,1]$, indicated by the red square.
\begin{figure}
\centering
\includegraphics[width=.5\linewidth]{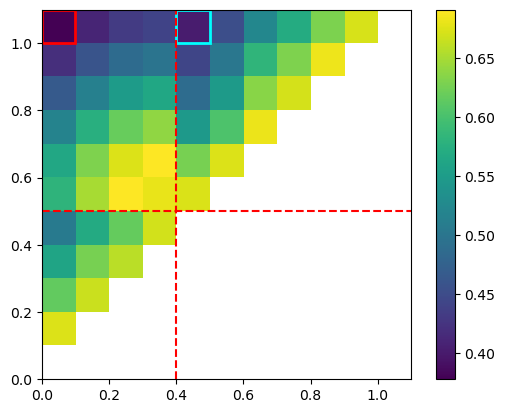}
    \caption{Counterexample motivating the CRS. The value of the cell with bottom-left corner at coordinate $(a,b)$ corresponds to the EPE of the region $[a, b]$. Regions below the horizontal dashed red line are subsets of the left subinterval $[0, c]$; regions to the right of the vertical dashed red line are subsets of the right subinterval $[c, 1]$; and regions in the upper-left quadrant (above the horizontal dashed red line and to the left of the vertical dashed red line) contain points in both the left and right subintervals. The interval $[c, 1]$ is the region of minimal EPE in which the Cox model holds (cyan boxed cell). However, the region which minimizes the EPE overall is the entire interval $[0,1]$ (red boxed cell), which does not follow a Cox model.}
    \label{fig: epe counterex}
\end{figure}

We remark that in spite of this example, we still believe the EPE to be useful for two reasons. First, although the Cox model is not well-specified on the minimizing region, it will still make good predictions.
Second,
it is not immediately clear how to turn the CRS into a single value for a group of points, so the EPE may be more practical for summarizing a collection of points.

\subsection{CRS Definition}
We restrict our attention to the case of uncensored data for now. For the Cox model with unconstrained baseline hazard function, all of the information is contained in the order of failures. Thus, we examine the probability of the rank statistics of the observed points, conditional on the estimated Cox coefficients and the observed failure order of the core group.

Let $\b$ be the fitted model coefficients and $x_1, \ldots, x_n$ be the feature vectors in the core group, labeled such that $t_1 < t_2 < \cdots < t_n$. For a ``test'' point with features $x^*$ and failure time $t^*$, we wish to compute the probability that the rank of $x^*$ is at least as extreme (high or low) as its observed value, conditional on the other observed failure times and assuming that $x^*$ follows the same Cox model as the core group. To do this, we work with the \emph{conditional rank statistics} of $x^*$, defined as:
\begin{equation} \label{eq: cond rank dist defn}
    r^c_k(x^*) = \P(t_{k-1} < t^* < t_k \: | \: x^*, x_1,\ldots,x_n; \: t_1<\cdots<t_n),
\end{equation}
where the probability is computed assuming each pair $(x, t)$ follows the same Cox model with fixed (unknown) baseline hazard function $\lambda_0(t)$ and Cox coefficients $\b$.
It will also be convenient to define the \emph{unconditional rank probabilities} of $x^*$ as
\begin{equation} \label{eq: uncond rank dist defn}
    r_k(x^*) = \P(t_1 < \cdots < t_{k-1} < t^* < t_k < \cdots < t_n \: | \: x^*, x_1,\ldots,x_n).
\end{equation}
By Bayes' rule, $r^c_k(x^*) = r_k(x^*)/(\sum_{j=1}^{n+1} r_j(x^*))$, so it suffices to compute the unconditional rank probabilities of $x^*$.
When the data are generated according to the Cox model, we have
\begin{equation} \label{eq: uncond rank prob}
    r_k(x^*) = \prod_{i=1}^{n+1} \frac{\mathrm{exp}(\b^\T x^{(k)}_i)}{\sum_{j=i}^{n+1} \mathrm{exp}(\b^\T x^{(k)}_j)},
\end{equation}
where we have defined $x_i^{(k)} = x_i$ if $i<k$, $x_i^{(k)} = x^*$ if $i=k$, and $x_i^{(k)} = x_{i-1}$ if $i>k$
(i.e., the $i$-th feature vector when $x^*$ has been ``inserted'' in the $k$-th position). Using the expression from Bayes' rule, we can then compute the conditional rank statistics for $x^*$.

\paragraph{Generalization to Censored Data} The conditional rank statistics have a straightforward generalization to the partial likelihood and censored data. We again consider the distribution of possible failure times for $x^*$ among all of the events (failure or censoring) experienced by the other points. If we know the actual rank of $x^*$ (i.e., if it failed), then we can conduct a two-tailed test after computing the statistics. If $x^*$ was censored, then we can only form a test based on its right tail.

Let $t_1 < \cdots < t_n$ be the event times for the points with features $x_1, \ldots, x_n$ in the core group, and let $\d_i$ be the corresponding failure indicators ($\d_i = \I\{x_i \textrm{ failed (was not censored) at time } y_i\}$). The partial likelihood that $x^*$ fails with event rank $k$ is
\begin{equation} \label{eq: cens uncond rank prob}
    r_k(x^*) = \prod_{i=1}^{n+1} \l( \frac{\mathrm{exp}(\b^\T x^{(k)}_i)}{\sum_{j=i}^{n+1} \mathrm{exp}(\b^\T x^{(k)}_j)} \r)^{\d_i} = \prod_{i \: : \: \d_i = 1} \l( \frac{\mathrm{exp}(\b^\T x^{(k)}_i)}{\sum_{j=i}^{n+1} \mathrm{exp}(\b^\T x^{(k)}_j)} \r),
\end{equation}
where $x_i^{(k)}$ are defined as before. Note that this is simply the standard Cox partial likelihood if $x^*$ fails as the $k$-th event. The conditional failure ``likelihoods'' $r^c_k(x^*)$ are then defined analogously to the case with no censoring, i.e., $r^c_k = r_k/\sum_{j=1}^{n+1}r_j$. We note that these are no longer actually probabilities or proper likelihoods in the presence of censoring.

\paragraph{Fast Implementation}
Computing the conditional rank probabilities naively is inefficient on large datasets, scaling as $\Omega(n^3)$. Using some recursive relationships between the unconditional rank probabilities, we can drastically reduce this runtime down to $O(n)$ which also leads to marked practical efficiency gains. Details can be found in Appendix~\ref{appendix: runtime}.

\subsection{Properties of the CRS}
The CRS is a powerful object for detecting deviations from the Cox model. However, its discrete and conditional nature makes it difficult to analyze directly. Here we provide a ``conditional Glivenko-Cantelli theorem'' for the CRS, which shows that in the absence of censoring, the CRS converges uniformly to its large-sample expectation in probability. This result allows us to analyze the effect size which the CRS can detect given a large sample.

\begin{restatable}{thm}{crsconv} \label{thm: crs convergence}
Fix a region $R$ of the feature space and assume there is no censoring in the data. Let 
\[
G(t) = \P(T(X) \leq t \: | \: X \in R)
\]
be the cdf of the marginal core group survival time distribution, i.e., where we first sample random features $X$ from the core group, then sample $T|X$. We fix a test point $x^*$ and assume that its survival time $T(x^*)$ is absolutely continuous with bounded Radon-Nikodym derivative with respect to this marginal core group survival time distribution. Let 
\[
\hat{G}_n(t) = \frac1n \sum_{i=1}^n \I\{t_i \leq t\}
\]
be the empirical cdf of the survival times given $n$ i.i.d. samples. In particular, in the case of no censoring,
\[
\P(\hat{G}_n(T(x^*)) \leq \a | \: x_1,\ldots, x_n, \: t_1 \leq \cdots \leq t_n) = \sum_{k=1}^{\a n} r^c_k(x^*)
\]
is precisely the $\alpha$ tail of the CRS for the test point $x^*$. Then as $n\rightarrow\infty$, uniformly in $\a$ we have
\[\P(\hat{G}_n(T(x^*)) \leq \a | \: x_1,\ldots, x_n, \: t_1 \leq \cdots \leq t_n) \stackrel{p}{\rightarrow} \P(G(T(x^*)) \leq \a).\]
\end{restatable}
We remark that the absolute continuity and bounded density of $T(x^*)$ with respect to the marginal core group survival distribution can be easily made to hold by imposing mild regularity conditions on the hazard function $\lambda(t, x)$ and the core group $R$.

\section{Algorithms} \label{sec: algs}
Having studied the theoretical underpinnings of the Cox subgroup discovery problem, we next turn to methodology. We introduce eight algorithms in total (not including the base method of just fitting the Cox model to the entire dataset). Given a training dataset and hyperparameter settings, each method returns a region $R$ which defines the subgroup, after which the Cox coefficients $\b$ are fit to the training data in this subgroup. In addition to the method descriptions, we also provide the hyperparameter ranges for each method. To ensure a fair comparison, each method has a total of 100 possible hyperparameter settings. More detail for the algorithms is given in Appendix~\ref{appendix: algs expanded}. 

\begin{itemize}
    \item \textbf{Base:} The simplest baseline is the standard Cox model fit to the entire dataset. The ``subgroup'' $R$ in this case is just the bounding box for the entire dataset. \emph{Hyperparameters:} None.

    \item \textbf{Random:} As another naive baseline, we also implement a method which randomly selects $2d$ points (where $d$ is the data dimension) and forms their bounding box to determine $R$. The number $2d$ is chosen so that each randomly selected point determines one side of the bounding box. \emph{Hyperparameters:} Random seed for selecting the $2d$ points $\in\{0, \ldots, 99\}$.

    \item \textbf{Survival Tree (ST):} Survival trees \citep{ishwaran2008survivalforest} are a classical tree-based method for survival analysis where the node splitting is performed according to the log-rank criterion. We use the regions described by the tree's leaves to define subgroups, and fit a Cox model in each leaf (similar to linear model trees \citep{wang1997lmt}). The splitting terminates when either (1) no split improves the log-rank criterion; (2) a maximum depth is reached; or (3) a minimum leaf size is reached, i.e., every valid split reduces the leaf size to below this threshold. The final subgroup $R$ is chosen as the leaf with minimum training EPE. \emph{Hyperparameters:} Maximum tree depth $\in \{1,2,\ldots,25\}$, minimum leaf size $\in \{5, 10, 20, 40\}$.

    \item \textbf{Cox Tree (CT):} Here we grow a tree using the EPE as a splitting criterion. Specifically, given a split of a parent node into left and right children, the impurity is defined as the weighted average of the EPE in the left and right children (weighted according to the size of the children). A split is chosen which minimizes this impurity and the tree building proceeds as normal.
    \emph{Hyperparameters:} Maximum tree depth $\in \{1,2,\ldots,25\}$, minimum leaf size $\in \{5, 10, 20, 40\}$.

    \item \textbf{PRIM:} The Patient Rule Induction Method (PRIM) introduced by \cite{friedman1999bump} is a general-purpose subgroup discovery method for finding axis-aligned regions of the data where a pre-specified quality function is maximized. We adapt this method to the survival analysis setting by using the negative EPE as the quality. \emph{Hyperparameters:} Peeling/pasting parameter $\a \in \{0.01, 0.02, \ldots, 0.25\}$, minimum support size $\b_0 \in \{0.005, 0.01, 0.02, 0.04\}$.

    \item \textbf{DDGroup (DG):} The DDGroup framework of \cite{izzo23subgroup} can naturally incorporate both the EPE and the CRS in a single algorithm, using the EPE for the core group selection procedure and the CRS for the rejection score/expansion phase. In particular, let $k^*$ be the rank of the failure time among core group failure times $t_1,\ldots,t_n$. We define the \emph{rank tail score}
    \[ \tau^* = \min\l\{\sum_{k=1}^{k^*} r^c_k(x^*), \: \sum_{k=k^*}^{n+1} r^c_k(x^*)\r\} \]
    and check whether $\tau^* <$ a specified rejection treshold to define the rejections labels. \emph{Hyperparameters:} Core group size $\in \{0.05n, 0.1n\}$, rejection threshold quantile $\a \in \{0.01, 0.02, \ldots, 0.5\}$.
    
    \item \textbf{C-Index DDGroup (DG-CI):} We also test three ablations of DDGroup. The first, DG-CI, uses the C-index to measure core group quality, i.e., the neighborhood with the highest C-index is selected as the core group. The C-index is also used to define conformity scores of test points to the core group during the point rejection phase of DDGroup. Specifically, given a core group $\{(x_i, t_i, \d_i)\}_{i=1}^k$, fit Cox coefficients $\b$, and a test point $(x^*, t^*, \d^*)$, the conformity score is
    \[
    s^*_{\mathrm{CI}} = \frac{\sum_{i \: : \: t_i < t^*, \, \d_i = 1} \I\{\b^\T x_i > \b^\T x^*\} + \d^* \sum_{i \: : \: t_i \geq t^*} \I\{\b^\T x_i \leq \b^\T x^*\}}{|\{i \: : \: t_i < t^*, \, \d_i = 1\}| + \d^*|\{i \: : \: t_i \geq t^*\}| },
    \]
    i.e., the fraction of core group points which are concordant with the test point and the given model. We reject points whose score $s^*_{\mathrm{CI}}$ falls below the $\a$ quantile of all scores computed over the dataset. \emph{Hyperparameters:} Core group size $\in \{0.05n, 0.1n\}$, rejection threshold quantile $\a \in \{0.01, 0.02, \ldots, 0.5\}$.

    \item \textbf{Partial likelihood DDGroup (DG-PL):} For this algorithm, we use the partial likelihood to implement the core group and rejection components of DDGroup. Specifically, the core group is selected as the neighborhood with the largest partial likelihood after fitting the Cox model. Given a core group $\{(x_i, t_i, \d_i)\}_{i=1}^k$, fit Cox coefficients $\b$, and a test point $(x^*, t^*, \d^*)$, the conformity score is
    \[
    s^*_{\mathrm{PL}} = \begin{cases}
        \frac{\exp(\b^\T x^*)}{\exp(\b^\T x^*) + \sum_{i \: : \: t_i \geq t^*} \exp(\b^\T x_i)}, &\quad \d^* = 1 \\[10pt]
        \sum_{i \: : \: t_i \geq t^*, \d_i = 1} \frac{\exp(\b^\T x^*)}{\exp(\b^\T x^*) + \sum_{j \: : \: t_j \geq t_i} \exp(\b^\T x_j)}, &\quad \d^* = 0
    \end{cases}.
    \]
    That is, the rejection score is the partial likelihood term for the test point when the test point is uncensored; otherwise, it is the sum of all possible partial likelihood terms which are consistent with a censored test observation. As with the other DDGroup versions, we reject points whose score $s^*_{\mathrm{PL}}$ is in the bottom $\a$ quantile of scores over the whole dataset. We also briefly remark that the double summation in the $\d^*=0$ case can lead to an $\Omega(k^2)$ computation when computed naively. We give an efficient implementation based on sorting the event times $y_i$ and using running partial sums which reduces the cost to $O(k\log k)$ when the times are unsorted, or $O(k)$ when the times are pre-sorted. \emph{Hyperparameters:} Core group size $\in \{0.05n, 0.1n\}$, rejection threshold quantile $\a \in \{0.01, 0.02, \ldots, 0.5\}$.
    
    \item \textbf{DDGroup without expansion (DG-NE):} An ablation of DDGroup without the expansion phase, where we just find the best core group (as measured by EPE) and do not expand. \emph{Hyperparameters:} Core group size $\in \{0.01n, 0.02n, \ldots, 1.0n\}$.
\end{itemize}
Pseudocode for each method is provided in Appendix~\ref{appendix: algs expanded}. Full implementations of each of these methods are provided in the project GitHub repo: \href{https://nj-gitlab.nec-labs.com/zach/cox-subgroup}{\texttt{https://nj-gitlab.nec-labs.com/zach/cox-subgroup}}.

\section{Performance Guarantees}
It is possible to show that DDGroup, which takes advantage of both the EPE and the CRS, recovers the correct region in a well-specified setting. Our proof relies on several assumptions:
\begin{enumerate}[label=A\arabic*.]
    \item The hazard function for the entire dataset has the form $\lambda(t; x) = \lambda_0(t)e^{h(x)}$ for some unknown risk function $h$. \label{assm: hazard func}
    \item There is no censoring in the data. \label{assm: censoring}
    \item There is a unique largest region $R^*$ which minimizes the EPE, and $R^*$ is an axis-aligned box.
    Conditional on $x\in R^*$, we have $h(x) = \b^\T x$ for some $\b$, i.e., the Cox model is well-specified. \label{assm: r star}
    \item The core group selection procedure finds a group of points which belong to $R^*$, and the Cox model fit to these points recovers the true parameters $\b$. \label{assm: core group}
    \item The error between the finite conditional rank statistics and its large-sample limit according to Theorem~\ref{thm: crs convergence} is negligible. We use this limiting distribution for the analysis, rather than the finite sample version described in Section~\ref{sec: crs}. \label{assm: crs}
\end{enumerate}
Under these assumptions, we can analyze a ``theoretically stylized'' version of DDGroup, given by Algorithms~\ref{alg: growing box} and \ref{alg: pipeline} in the appendix. Compared to the practical implementation, the main differences are:
\begin{itemize}
    \item Using the large-sample limit $G$ of the CRS to reject points as opposed to the finite sample version (Assumption~\ref{assm: crs});
    \item Using approximate, rather than exact, quantiles of $G$ to reject points;
    \item Splitting the training data so that disjoint subsets are used for finding the core group and performing the rejection/expansion step to avoid introducing intricate dependencies between the survival times (the practical implementation uses the entire training set for both steps);
    \item Assuming $\hat{\b} = \b$ when fit to the core group (Assumption~\ref{assm: core group}); and
    \item The inclusion of an additional hyperparameter, the expansion speed of the sides of the growing box, as in the theory for the original DDGroup \citep{izzo23subgroup}.
\end{itemize}
Our main theorem shows that this version of DDGroup can recover the ground truth region $R^*$ with high probability, given a large enough effect size. The full proof of this theorem, as well as a discussion on the validity of the assumptions and the strength of the convergence result, can be found in Appendix~\ref{appendix: proofs}.
\begin{restatable}{thm}{main} \label{thm: main}
Let $\hat{R}_n$ be the region output by Algorithm~\ref{alg: pipeline} on a dataset of $2n$ i.i.d. points satisfying the above assumptions. For any $\e > 0$, there is an effect size $C_\e = O(\log\e^{-1})$ such that if $|h(x) - \b^\T x| \geq C_\e$ outside of $R^*$, then there exist settings for the hyperparameters of Alg.~\ref{alg: pipeline} such that with probability at least 0.99, 
\[
R^* \subseteq \hat{R}_n \hspace{.25in} \textrm{and} \hspace{.25in} \mathrm{vol}(\hat{R}_n \setminus R^*) \leq C'\e
\]
for another constant $C'$ as $n\rightarrow\infty$.
\end{restatable}

\section{Experiments} \label{sec: experiments}
We now test our methods empirically.
Except when noted otherwise, in all of the following experiments, we used the following setup.
First, each method was supplied with a range of 100 total hyperparameter settings (specified in Section~\ref{sec: algs}). Given a hyperparameter setting and a training dataset, each method returns a (subgroup, Cox coefficients) pair. Thus, for each random train/test split of a given dataset, each method returns a list of at most 100 subgroups and associated Cox coefficients.\footnote{For some hyperparameter settings of some methods, an error may occur and the method will fail to return a region. This can happen if, e.g., not enough uncensored data is included in a region to successfully fit a Cox model. Thus there are at most 100 subgroups returned, rather than exactly 100.} From among these 100 subgroups, we filtered to those which contained at least 10\% of the training data to prevent overfitting. The remaining subgroup with the lowest training EPE was selected for each method. This subgroup and associated Cox model was then used to compute the EPE, rejection fraction, and C-index on the test data, and the results were averaged over 10 random train/test splits of the data.

\subsection{Synthetic Data} \label{sec: synth}
\paragraph{Metrics}
For the synthetic data, we have access to the ground truth region $R^*$ and can use it as the primary evaluation for each method. To this end, we follow \cite{izzo23subgroup} and define the precision and recall of an estimated region $R$ by
\[
\textrm{Precision} = \frac{\mathrm{vol}(R\cap R^*)}{\mathrm{vol}(R)}, \quad \quad \textrm{Recall} = \frac{\mathrm{vol}(R\cap R^*)}{\mathrm{vol}(R^*)}.
\]
The F1 score is then defined in the usual manner as the harmonic mean of precision and recall. An F1 score closer to 1 means better recover of the ground truth, with perfect recovery iff F1 = 1.

\paragraph{Dataset Descriptions}
We test on two synthetic datasets: Synth-Counter and Synth-Nonlinear. Synth-Counter is exactly the dataset described in Section~\ref{sec: epe counterex}. We use this to test whether the CRS can overcome some of the shortcomings of the EPE. In Synth-Nonlinear, the features are uniformly generated from $B = [-1,1]^d$, and there is a ground-truth region $R^*$ which is an axis aligned box, and conditional on $x\in R^*$ the survival time $y$ is generated according to a Cox model with baseline hazard $\lambda_0(t) \equiv 1$ and some ground truth Cox coefficients $\b^*$. Equivalently, $t|x$ is exponentially distributed with rate $e^{x^\T \b^*}$. For $x\not\in R^*$, $t|x$ also follows a Cox model but with respect to a nonlinear transformation of the features $x$.

\paragraph{Results: Synth-Counter}
The purpose of this experiment is to show that the CRS can cope with the counterexample discussed in Section~\ref{sec: epe counterex}, while the C-index and partial likelihood cannot. In this dataset, the region of minimum EPE in which a Cox model is well-specified is $R^* = [0.4, \, 1]$, but the region with minimum EPE overall is the entire interval $\cB = [0, 1]$. For this dataset only, rather than selecting the region with minimum training EPE for each method (which would generally return the whole interval $\cB$), we simply select the region with the highest F1 score among the 100 regions returned by the method for each of the 10 replicates, then average over the 10 replications. The results are shown in Table~\ref{tab: synth-counter}. We see that by leveraging the CRS, it is possible to tune DDGroup so that it recovers the ground truth region $R^*$. On the other hand, relying only on the C-index or partial likelihood, it is not possible to tune these ablations of DDGroup to recover $R^*$.

\begin{table}
    \centering
    \caption{Best F1 score obtained by different implementations of DDGroup in the example from Fig.~\ref{fig: epe counterex}. By using the CRS, it is possible to tune DDGroup to find the correct subgroup. Even with the best possible hyperparameter tuning, the other versions which do not take advantage of the EPE and CRS cannot recover $R^*$.} \label{tab: synth-counter}
    \begin{tabular}{l|cccc}
        \toprule
        Metric & Base & DG-PL & DG-CI & DG \\
        \midrule
        F1 ($\uparrow$) & 0.75 (0.00) & 0.81 (0.10) & 0.76 (0.04) & 0.94 (0.01) \\
        \bottomrule
    \end{tabular}
\end{table}

\paragraph{Results: Synth-Nonlinear}
With this dataset, we study the more common scenario where we will select a subgroup according to the minimum training EPE.
As before, DDGroup has the best performance in terms of ability to recover the ground truth $R^*$. With the exception of the survival tree and PRIM, all methods exhibit significant improvement over the baseline (non-subgroup) approach in terms of F1 score, EPE value, and C-index. We note that although DDGroup obtains the best region recovery, DDGroup without the expansion phase obtains a lower EPE. We observed that this is due to DDG-NE selecting a region which is a proper subset of $R^*$ (precision $=1$, but lower recall). This conservative choice leads to slightly improved EPE at the cost of full recovery of $R^*$. In contrast, the full version of DDGroup---which optimizes for good EPE and reasonable fit to the Cox model simultaneously---is less conservative and recovers the full region.
\begin{table}[ht]
\centering
\caption{Comparison of methods on F1, Test EPE, and Test C-index. DDGroup (DG) obtains the best performance in terms of recovering $R^*$. The other methods generally have worse performance, though they all perform better than the baseline (non-subgroup) approach with the exception of the survival tree and PRIM.}
\label{tab: synth-nonlinear}
\begin{tabular}{l|ccc}
\toprule
Method & F1 ($\uparrow$) & Test EPE ($\downarrow$) & Test C-index ($\uparrow$) \\
\midrule
\rowcolor{gray!10}Base             & 0.29 (0.00) & 0.69 (0.00) & 0.54 (0.00) \\
\midrule
                  ST               & 0.28 (0.03) & 0.66 (0.01) & 0.59 (0.01) \\
\rowcolor{gray!10}PRIM             & 0.30 (0.01) & 0.69 (0.00) & 0.54 (0.00) \\
                  CT               & 0.33 (0.08) & 0.64 (0.04) & 0.60 (0.03) \\
\rowcolor{gray!10}DDG-PL           & 0.46 (0.03) & 0.64 (0.02) & 0.61 (0.02) \\
                  Random           & 0.75 (0.04) & 0.56 (0.03) & 0.74 (0.02) \\
\rowcolor{gray!10}DDG-CI           & 0.89 (0.03) & 0.39 (0.02) & 0.86 (0.01) \\
                  DDG-NE           & 0.90 (0.01) & \cellcolor{green!20}0.29 (0.01) & \cellcolor{green!20}0.88 (0.00) \\
\rowcolor{gray!10}\textbf{DDGroup} & \cellcolor{green!20}0.97 (0.01) & 0.38 (0.02) & 0.87 (0.01) \\
\bottomrule
\end{tabular}
\end{table}

\subsection{Real Data} \label{sec: real}

\paragraph{Additional Metrics}
Next, we perform experiments on real data to show the practical utility of our methods.
As discussed in Section~\ref{sec: crs}, practitioners may be interested in not only high predictability of the data (represented by low EPE) but also a qualitatively good fit of the Cox model to the data. When a ground truth region in which a Cox model is well-specified is unknown, we must resort to an observable proxy. Thus, we define the \emph{rejection fraction}, which measures the number of individual points in a subgroup which deviate from the Cox model based on the CRS. See Appendix~\ref{appendix: experiments} for the definition.

\paragraph{Datasets}
We used datasets from the sksurv Python package \citep{sksurv} as well as the METABRIC dataset \citep{curtis2012metabric}.
Each of these datasets studies patient survival under various conditions. A common design choice in clinical statistical analyses is to adjust for a single covariate to observe its effect, and generally simple subgroup definitions are preferred \citep{friedman2015fundamentals}. Thus, in each of these datasets, we adjust for a single covariate and define a subgroup in terms of age. We chose age as the subgroup-defining variable as age is known to modulate/interact with other biomarkers for many health-related outcomes \citep{belloy2023alzage,  mak2023cancerage, moqri2023age}. In Table~\ref{tab: real}, the results are presented as (base dataset name)-(adjusted covariate). Dataset descriptions are provided in Appendix~\ref{appendix: datasets}.

\paragraph{Results}
The results are shown in Table~\ref{tab: real}. We include a version of the table with standard errors in Appendix~\ref{appendix: extra results}.
In 5/7 datasets, DDGroup finds subgroups with the lowest EPE, indicating good predictive performance and few confidently inaccurate predictions. DDGroup also tends to have a low rejection fraction, indicating a small fraction of ``outliers'' within the subgroup which have very low likelihood of occurrence if the Cox model is correct. In fact, DDGroup is on the Pareto frontier of (EPE, rejection fraction) for every dataset, i.e., there is no method which is strictly better than DDGroup in terms of both EPE and rejection fraction for any dataset.

The results also show that while the EPE and C-index are indeed correlated, the correlation is not perfect (e.g., GBSG2-tsize). We believe that the C-index is still a useful metric to report in practice, especially as it is likely easier to interpret for non-experts. Nevertheless, when model calibration is an important consideration (as it often will be in high-stakes clinical settings), the EPE provides a more fine-grained metric which explicitly accounts for model confidence, and which can detect confident but inaccurate predictions which may be missed by the C-index.

We note that the EPE of a null model (i.e., a Cox model with $\b=0$, which gives any pair of units a 50-50 chance of which will fail first, i.e., a random guess) is $-\log \frac12 \approx 0.6931$. In some of the datasets, the baseline Cox model fit to the entire dataset, as well as some of the baselines, perform no better than a random model.

While DDGroup generally has the best performance, we also find that the Cox tree can be useful in some scenarios. CT finds subgroups which have reduced EPE compared to the baseline in 5/7 datasets. While the best subgroup found by CT is generally of lower quality than DDGroup, CT has the advantage of providing a model for the entire dataset, as the entire feature space is divided into the leaves of the tree, each of which has a local Cox model. The survival tree also has this feature, though its performance is generally inferior to the Cox tree. Thus, when a global model is desired, the Cox tree may be a useful alternative.

\begin{table}[h!]
    \centering
    \caption{Results on the real datasets. The best value for each metric is highlighted in green for each dataset. DDGroup tends to find the subgroup with lowest EPE and offers a good balance between EPE and Cox model fit. The other baselines are often capable of improving of the basic Cox model fit to the entire data, though in some cases they do not perform better than a null model. While the Cox tree obtains higher EPE than DDGroup on average over these datasets, the fact that it returns a global interpretable model may still make it attractive in some cases.}
    \resizebox{\textwidth}{!}{
    \begin{tabular}{ll|ccccccccc}
        \toprule
        & & Base & Rand & PRIM & ST & CT & DG-PL & DG-CI & DG-NE & DG \\
        \midrule
        \rowcolor{gray!10}AIDS-CD4 & EPE ($\downarrow$) & 0.54 & 0.59 & 0.53 & 0.58 & 0.50 & 0.54 & 0.59 & 0.46 & \cellcolor{green!20}0.43 \\
        &Rej@10\% ($\downarrow$) & 0.07 & 0.11 & 0.06 & 0.08 & 0.12 & 0.08 & 0.10 & \cellcolor{green!20}0.05 & 0.07 \\
        \rowcolor{gray!10}&C-Index ($\uparrow$) & 0.71 & 0.72 & 0.73 & 0.71 & 0.74 & 0.72 & 0.68 & 0.75 & \cellcolor{green!20}0.76 \\
        &Size ($\uparrow$) & 1.00 & 0.17 & \cellcolor{green!20}0.87 & 0.32 & 0.17 & 0.73 & 0.55 & 0.14 & 0.20 \\
        \midrule
        \rowcolor{gray!10}AIDS-Karnof&EPE ($\downarrow$) & 0.62 & 0.58 & 0.62 & 0.64 & 0.72 & 0.62 & 0.72 & 0.77 & \cellcolor{green!20}0.38 \\
        &Rej@10\% ($\downarrow$) & 0.09 & \cellcolor{green!20}0.01 & 0.15 & 0.11 & 0.19 & 0.07 & 0.06 & 0.07 & 0.01 \\
        \rowcolor{gray!10}&C-Index ($\uparrow$) & 0.66 & 0.68 & 0.66 & 0.64 & 0.68 & 0.66 & 0.69 & 0.67 & \cellcolor{green!20}0.84 \\
        &Size ($\uparrow$) & 1.00 & 0.13 & \cellcolor{green!20}0.96 & 0.21 & 0.44 & 0.82 & 0.65 & 0.13 & 0.15 \\
        \midrule
        \rowcolor{gray!10}AIDS-prior&EPE ($\downarrow$) & 0.70 & 0.67 & 0.70 & 0.69 & 0.68 & 0.66 & 0.70 & 0.67 & \cellcolor{green!20}0.65 \\
        &Rej@10\% ($\downarrow$) & 0.14 & 0.21 & 0.06 & 0.06 & 0.06 & 0.11 & 0.14 & \cellcolor{green!20}0.05 & 0.13 \\
        \rowcolor{gray!10}&C-Index ($\uparrow$) & 0.46 & 0.55 & 0.46 & 0.56 & 0.57 & 0.52 & 0.46 & 0.57 & \cellcolor{green!20}0.58 \\
        &Size ($\uparrow$) & 1.00 & 0.16 & 0.94 & 0.10 & 0.16 & 0.64 & \cellcolor{green!20}1.00 & 0.12 & 0.16 \\
        \midrule
        \rowcolor{gray!10}GBSG2-tsize&EPE ($\downarrow$) & 0.68 & 0.62 & 0.68 & 0.68 & 0.63 & 0.68 & 0.65 & 0.61 & \cellcolor{green!20}0.61 \\
        &Rej@10\% ($\downarrow$) & 0.15 & 0.13 & 0.08 & \cellcolor{green!20}0.04 & 0.05 & 0.13 & 0.11 & 0.05 & 0.07 \\
        \rowcolor{gray!10}&C-Index ($\uparrow$) & 0.57 & 0.62 & 0.55 & 0.58 & 0.62 & 0.58 & \cellcolor{green!20}0.64 & 0.64 & 0.64 \\
        &Size ($\uparrow$) & 1.00 & 0.15 & \cellcolor{green!20}0.94 & 0.26 & 0.11 & 0.90 & 0.39 & 0.12 & 0.12 \\
        \midrule
        \rowcolor{gray!10}MBRIC-MKI&EPE ($\downarrow$) & 0.69 & 0.70 & 0.69 & 0.70 & 0.72 & 0.69 & 0.70 & 0.72 & \cellcolor{green!20}0.69 \\
        &Rej@10\% ($\downarrow$) & 0.08 & 0.12 & \cellcolor{green!20}0.08 & 0.14 & 0.13 & 0.08 & 0.09 & 0.15 & 0.11 \\
        \rowcolor{gray!10}&C-Index ($\uparrow$) & 0.49 & 0.50 & 0.49 & 0.52 & 0.49 & 0.49 & 0.49 & 0.49 & \cellcolor{green!20}0.54 \\
        &Size ($\uparrow$) & 1.00 & 0.17 & 0.84 & 0.26 & 0.25 & \cellcolor{green!20}1.00 & 0.64 & 0.11 & 0.38 \\
        \midrule
        \rowcolor{gray!10}VLC-Karnof & EPE ($\downarrow$) & 0.58 & 0.42 & 0.58 & 0.45 & 0.22 & 0.32 & 0.34 & \cellcolor{green!20}0.19 & 0.33 \\
        &Rej@10\% ($\downarrow$) & \cellcolor{green!20}0.04 & 0.12 & 0.03 & 0.08 & 0.21 & 0.11 & 0.10 & 0.20 & 0.09 \\
        \rowcolor{gray!10}&C-Index ($\uparrow$) & 0.69 & 0.84 & 0.68 & 0.77 & 0.93 & 0.92 & 0.87 & \cellcolor{green!20}0.93 & 0.87 \\
        &Size ($\uparrow$) & 1.00 & 0.17 & \cellcolor{green!20}0.88 & 0.23 & 0.20 & 0.23 & 0.22 & 0.18 & 0.28 \\
        \midrule
        \rowcolor{gray!10}WHAS-DBP&EPE ($\downarrow$) & 0.67 & 0.83 & 0.67 & 0.65 & 0.64 & 0.67 & \cellcolor{green!20}0.55 & 0.74 & 0.62 \\
        &Rej@10\% ($\downarrow$) & 0.07 & \cellcolor{green!20}0.01 & 0.13 & 0.21 & 0.16 & 0.04 & 0.02 & 0.20 & \cellcolor{green!20}0.01 \\
        \rowcolor{gray!10}&C-Index ($\uparrow$) & 0.61 & 0.50 & 0.61 & 0.68 & 0.66 & 0.64 & \cellcolor{green!20}0.75 & 0.63 & 0.70 \\
        &Size ($\uparrow$) & 1.00 & 0.15 & \cellcolor{green!20}0.98 & 0.15 & 0.16 & 0.62 & 0.20 & 0.15 & 0.13 \\
        \bottomrule
    \end{tabular}
    }
    \label{tab: real}
\end{table}

\paragraph{Case Study: NASA Jet Engine Data}
We performed an additional in-depth case study on the simulated jet engine failure data from \cite{saxena2008damage} to show the utility of the interpretable subgroups discovered by our methods. (The data are available through \href{https://www.kaggle.com/datasets/behrad3d/nasa-cmaps}{Kaggle}.) The goal of the dataset is to predict the remaining useful life (RUL) of a jet engine as a function of different sensor measurements and operating conditions. The data are originally presented as a time series with different sensor measurements recorded as the engine is used. We converted the data into a standard survival analysis format by using each collection of sensor readings as a fixed covariate vector, then using the remaining time until failure from that reading as the survival time. In this dataset, all failures are observed, so there is no censoring.

As discussed by \cite{saxena2008damage}, in this dataset, there are known operating conditions (related to different altitudes, flight speeds, and air temperatures) where the relationship between sensor readings and RUL will differ qualitatively. Thus, the different operating conditions serve as natural subgroups of the data which we can try to recover. Note that the data are created by simulating physical processes leading to the fault of the engine; in particular, the simulation has nothing to do with the Cox model at face value, thereby testing the robustness of our methods to misspecification.

To create the dataset, we subsampled data where the flight speed and air temperature were in a fixed state, with the goal to recover the two possible altitude operating conditions. The creation of a new datapoint from every time point in the original time series leads to a very large dataset. Thus, from among the remaining data, we further randomly subsampled down to $n=5000$ datapoints to avoid excessive runtimes. As in the other experiments, we perform an 80-20 train/test split and average all results over 10 such splits.

The Cox model adjusts for two sensor readings: the engine core speed (``Nc,'' in RPM) and the corrected fan speed (``NRf,'' also in RPM).\footnote{Note: In modern jet engines, the core and fan are two separate components which can rotate at different speeds, leading to potentially different values for these covariates. See \href{https://www.grc.nasa.gov/WWW/K-12/airplane/aturbf.html}{https://www.grc.nasa.gov/WWW/K-12/airplane/aturbf.html} for an explanation.} Since the data were not uniformly distributed over the covariate space, we measure the precision and recall of a discovered subgroup according to the number of datapoints in the ground truth and estimated regions, rather than the region volumes. Specifically, given a ground truth region $R^*$, an estimate $\hat{R}$, and the dataset of features $\{x_i\}_{i=1}^n$, we define
\[
\textrm{Precision}(\hat{R}, R^*) = \frac{\# \{i \in [n] \: : \: x_i \in \hat{R} \cap R^*\}}{\# \{i \in [n] \: : \: x_i \in \hat{R}\}}, \quad \quad \textrm{Recall}(\hat{R}, R^*) = \frac{\# \{i \in [n] \: : \: x_i \in \hat{R} \cap R^*\}}{\# \{i \in [n] \: : \: x_i \in R^*\}}.
\] 
In the dataset, there are two potential ground truth regions $R^*_{\mathrm{lo}}$ and $R^*_{\mathrm{hi}}$ corresponding to lower and higher altitude flights. Given an estimate $\hat{R}$, we associate it to whichever of the two ground truth subgroups yields a higher precision and recall.

The results are shown in Table~\ref{tab: nasa}.
Several of the methods find subgroups with perfect precision. This means that the subgroup belongs entirely to one of the two possible operating conditions, which is ideal; there is complete separation between the two qualitatively different regimes.
The Cox tree performs best on this dataset, obtaining a low EPE value and the highest recall among methods with perfect precision. Some of the baseline methods (Random and the C-index ablation of DDGroup) find subgroups with slightly lower EPE than DDGroup, at the cost of identifying a smaller fraction of the ground truth region (lower recall).
We remark that because the Cox model is not necessarily a perfect fit to this data (which was generated via simulations of physical equations), we do not expect that all of the data within a single operating regime will be grouped together by the Cox model.

The methods which obtain perfect precision all return subsets of the high altitude operating conditions. Interestingly, there is a \emph{qualitative} as well as quantitative change in the Cox coefficients for these two settings. Without accounting for the subgroup, both coefficients have small positive values. However, for the high altitude subgroups, the coefficient on the core speed has a larger positive value, while the coefficient on fan speed becomes negative. There is a natural physical interpretation of this observation. On average, increased usage of both the fan and core components will lead to degradation of the engine due to mechanical wear; this is reflected by the small positive coefficients of the Cox model fit to the whole dataset. However, by taking a subgroup approach, we can consider a \emph{different baseline hazard function} for different subgroups. High altitude will likely require greater engine utilization to maintain overall. However, conditional at flying at high altitude, a greater usage of the fans may lead to less strain on the engine than relying on the core. Thus, the subgroup analysis approach suggests an actionable plan for engine usage at high altitudes: increase reliance on the fan, and reduce reliance on the core if possible. This indeed seems to be the trend in modern jet engines which have a reduced core size (see, e.g., \href{https://www.nasa.gov/aeronautics/smaller-is-better-for-jet-engines/}{https://www.nasa.gov/aeronautics/smaller-is-better-for-jet-engines/}). We remark that as with any data-driven approach to hypothesis generation, these conclusions should be verified with domain experts and checked experimentally.

\begin{table}
\centering
\caption{Results on the NASA engine RUL dataset. The best EPE and C-index values are highlighted in green, as well as the methods which achieve perfect precision. The highest recall among those methods with perfect precision (Cox tree) is also highlighted in green. The Cox tree excels in terms of most metrics for this dataset. Several methods are able to discover subgroups which agree with the known ground truth operating conditions and which reveal qualitatively different behavior from the overall dataset.} \label{tab: nasa}
\resizebox{\textwidth}{!}{
\begin{tabular}{l|ccccccccc}
\toprule
 & Base & Rand & PRIM & ST & CT & DG-PL & DG-CI & DG-NE & DG \\
\midrule
\rowcolor{gray!10}EPE ($\downarrow$) & 0.65 (0.00) & 0.53 (0.01) & 0.65 (0.00) & 0.53 (0.01) & \cellcolor{green!20}0.52 (0.01) & 0.64 (0.01) & 0.54 (0.01) & 0.51 (0.02) & 0.53 (0.01) \\
C-Index ($\uparrow$) & 0.60 (0.00) & 0.735 (0.01) & 0.60 (0.00) & 0.73 (0.01) & 0.74 (0.01) & 0.62 (0.01) & 0.73 (0.01) & \cellcolor{green!20}0.75 (0.01) & 0.73 (0.01) \\
\rowcolor{gray!10}Precision ($\uparrow$) & 0.63 (0.01) & \cellcolor{green!20}1.00 (0.00) & 0.64 (0.01) & \cellcolor{green!20}1.00 (0.00) & \cellcolor{green!20}1.00 (0.00) & 0.67 (0.04) & \cellcolor{green!20}1.00 (0.00) & \cellcolor{green!20}1.00 (0.00) & \cellcolor{green!20}1.00 (0.00) \\
Recall ($\uparrow$) & 1.00 (0.00) & 0.24 (0.02) & 0.91 (0.06) & 0.40 (0.06) & \cellcolor{green!20}0.73 (0.10) & 0.92 (0.08) & 0.29 (0.04) & 0.16 (0.01) & 0.36 (0.08) \\
\rowcolor{gray!10}$\b_{\mathrm{Nc}}$  & 0.31 (0.01) &  2.35 (0.10) & 0.32 (0.01) &  2.00 (0.09) &  2.19 (0.08)                  & 0.53 (0.22) &  2.17 (0.09) &  2.51 (0.05) &  2.23 (0.09) \\
                  $\b_{\mathrm{NRf}}$ & 0.30 (0.01) & -1.19 (0.07) & 0.30 (0.01) & -0.94 (0.10) & -1.08 (0.05) & 0.15 (0.15) & -1.06 (0.06) & -1.28 (0.04) & -1.10 (0.06)\\
\bottomrule
\end{tabular}
}
\end{table}

\section{Related Work} \label{sec: related work}
The Cox model \citep{cox1972regression, cox1975partial} is a standard method for survival analysis, and it has found widespread use in practice due to its ease of interpretation. Nevertheless, the interpretability comes at the cost of strong modeling assumptions which may be violated in practice \citep{hernan2010hazards}.
Within the machine learning community, there has been a great deal of effort to apply modern ML techniques to survival data and provide more powerful and flexible models \citep{katzman2018deepsurv, hu2023conditional, wu2023frailty, bleistein2024dynamic, bertsimas2022trees, kim2023survival, chen2024kernets, lee2018deephit, rindt2022survival, vauvelle2024differentiable, wang2019mlsurv}.

One of the most common evaluation metrics for survival models is Harrell's concordance index (C-index) \citep{harrell1982cindex}, which evaluates a model according to how well the predicted failure order of the units matches the data. There are many known shortcomings of relying solely on the C-index for survival model evaluation \citep{hartman2023pitfalls}, and the examples raised in this paper add to the body of evidence that these are valid concerns. Thus, in addition to improving modeling flexibility, there has been recent interest in proper evaluation metrics for survival models \citep{yanagisawa2023proper, haider2020effective, qi2023effective}.

Our work sits at the intersection of two orthogonal topics, survival analysis and subgroup discovery. At a high level, subgroup discovery refers to mining datasets for subsets or regions in which the data distribution is in some sense ``interesting,'' usually quantified by a numerical score function taking an extreme value when evaluated on the subgroup \citep{friedman1999bump, atzmueller2015subgroup, leman2008exceptional, izzo23subgroup, xu2024syflow}.
While subgroup discovery is a general problem, it has found a great deal of applications in biostatistics \citep{lipkovich2017tutorial, lipkovich2023modern}. Many methods have been proposed to study heterogeneous treatment effects in patient populations, in particular to find patient groups which experience enhanced benefit from a treatment \citep{kehl2006responder, lipkovich2011subgroup, dusseldorp2014qualitative, lipkovich2014strategies, lipkovich2017subgroup, schnell2018multiplicity, schnell2021monte, li2023statistical}; 
or for purposes of patient stratification \citep{polonik2010prim, chen2015prim, huang2017patient}. 

The work most closely related to ours in spirit is \cite{wei2018change}, which also studied a subgroup discovery problem for the Cox model. The resulting model, called the change-plane Cox model, defines subgroups via the two sides of a hyperplane rather than an axis-aligned box. There are also more restrictive assumptions on the relationship between the two resulting subgroups in their setting than in the setting considered in the present paper. These differences are quite significant and render their method unsuitable for our setting; we give an in-depth discussion in Appendix~\ref{appendix: algs expanded}. \cite{zhang2025changeplanecox} also studied properties of this change-plane Cox model.

\section{Conclusion} \label{sec: conclusion}
We introduced the problem of subgroup discovery with the Cox model, where the goal is to find interpretable subsets of the feature space in which a Cox model makes confident and accurate predictions, and proposed several methods for this problem. The most successful methods rely on two components: the expected prediction entropy (EPE), which quantifies the ability of a survival model to discriminate between the relative risk of failure for two units; and the conditional rank statistics (CRS), a statistical object which can be used to measure the deviation of an individual datapoint to the distribution of survival times in an existing subgroup. We studied the some theoretical properties of these methods and metrics and confirmed their effectiveness empirically on synthetic and real datasets.

\textbf{Limitations \& Future Work}
As our work is the first to address subgroup discovery with the Cox model, there are many open directions for follow-up work.
On the metrics side, we derived some important properties of the EPE, but we also note a drawback: there is an intricate dependence between the feature distribution and the value of the EPE. Devising a way to disentangle the effect of the feature distribution from the intrinsic performance of the model itself on the EPE (or developing another metric without this problem) would provide better ``apples-to-apples'' comparisons between the value of different subgroups. 
Another important open direction is to create valid p-values for Cox models and subgroups discovered with out methods. A naive Bonferroni correction may be overly conservative when each method produces hundreds of possible subgroups. In this paper, we have partially alleviated this issue by using a held-out test set, but more advanced methods tailored to the subgroup problem may give more precise control of false discoveries.

\bibliographystyle{plainnat}
\bibliography{arxiv}

\newpage
\appendix

\section{Runtime Improvements} \label{appendix: runtime}
A naive implementation of the conditional rank tail probability took over 20 seconds to evaluate on a single point in some early experiments. Thus, a faster implementation is necessary. We will use the abbreviation $r_k = r_k(x^*)$.

First, we observe that the naive computation of a single $r_k$ from equation~\eqref{eq: uncond rank prob} will require $\Omega(n^2)$ time. This can easily be reduced to $O(n)$ by updating the partial sum contained in the denominator as each term in the product is computed, rather than recomputing it from scratch each time. With this modification, we can compute $r_1$ in $O(n)$ time.

We can obtain another speedup by computing the remaining $r_k$ recursively, rather than repeatedly using the procedure above from scratch for each $r_k$. A direct calculation using the formula~\eqref{eq: cens uncond rank prob} shows that
\begin{equation} \label{eq: cens rk ratio}
    r_{k+1} = \frac{(1-\d_k) e^{\b^\T x^*} + S_k}{e^{\b^\T x^*} - e^{\b^\T x_k} + S_k} \cdot r_k,
\end{equation}
where we have defined $S_k = \sum_{i=k}^n e^{\b^\T x_i}$. Again using the running partial sum trick to quickly compute $S_k$ (rather than computing from scratch each time), we can compute the next $r_{k+1}$ in constant time using the previous one. This means that $r_1,\ldots,r_{n+1}$ can \emph{all} be computed using only $O(n)$ time total.

The pseudocode for the resulting procedure is given in Algorithm~\ref{alg: fast cens log rank}. We have replaced the rank probabilities $r_k$ with the logarithms since when working with large datasets, working directly with the product of many probabilities (even when each is individually of ``reasonable'' size) can lead to numerical issues. Given the set of $\log r_k$, the CRS $r^c$ can then be computed by taking a softmax.

\begin{algorithm}
\caption{Fast computation of the log rank probabilities with censoring} \label{alg: fast cens log rank}
\begin{algorithmic}
\STATE $S \gets \sum_{i=1}^n e^{\b^\T x_i}$
\STATE log\_prod $\gets \b^\T x^* - \log(S + e^{\b^\T x^*})$
\FOR{$i = 1, \ldots, n$} 
    \STATE log\_prod $\gets$ log\_prod $+ \: \d_i(\b^\T x_i - \log S)$
    \STATE $S \gets S - e^{\b^\T x_i}$
\ENDFOR
\STATE $\log r_1 \gets$ log\_prod
\STATE
\STATE $S \gets \sum_{i=1}^n e^{\b^\T x_i}$
\FOR{$k = 1, \ldots, n$}
    \STATE $\log r_{k+1} \gets \log r_k + \log(S + (1-\d_k)e^{\b^\T x^*}) - \log(S + e^{\b^\T x^*} - e^{\b^\T x_k})$
    \STATE $S \gets S - e^{\b^\T x_k}$
\ENDFOR
\STATE
\RETURN $\log r_1, \ldots, \log r_{n+1}$
\end{algorithmic}
\end{algorithm}

\section{EPE Decreases with Region Size}
In this section, we give the proof of Theorem~\ref{thm: decreasing epe}.
\begin{lem} \label{thm: peaked ce}
Let $H(z) = -\frac{1}{1+e^{-z}} \log \frac{1}{1+e^{-z}} - \frac{1}{1+e^{z}} \log \frac{1}{1+e^{z}}$. Then $H(z)$ is a decreasing function of $|z|$.
\end{lem}
\begin{proof}
By taking a derivative, we see that the function $h(p) := -p\log p - (1-p)\log(1-p)$ is increasing in $p$ for $p\in[0,1/2]$ and decreasing in $p$ for $p\in[1/2, 1]$. Equivalently, $h(p)$ is a decreasing function of $|p-1/2|$. Setting $p(z)=\frac{1}{1+e^{-z}}$, we see that $p(z)=1/2$ for $z=0$ and $p(z)$ moves away from $1/2$ (i.e., $|p(z)-1/2|$ increases) as the magnitude of $z$ increases. This yields the desired result.
\end{proof}

\begin{lem} \label{thm: spreading ce}
Let $Z, Z'$ be random variables with densities $f, g$ respectively. Suppose that there exists a constant $c\geq0$ such that $f(z) \geq g(z)$ whenever $|z| < c$ and $f(z) \leq g(z)$ whenever $|z| \geq c$. Then $\E[H(Z)] \geq \E[H(Z')]$.
\end{lem}
\begin{proof}
We have the following:
\begin{align}
    \E[H(Z)] &- \E[H(Z')] = \int_{|z|< c} H(z) (f(z)-g(z)) \, dz + \int_{|z|\geq c} H(z) (f(z)-g(z)) \, dz \nn \\
    &\geq \inf_{|z| < c} H(z) \cdot \int_{|z|< c} (f(z)-g(z)) \, dz - \sup_{|z| \geq  c} H(z) \cdot \int_{|z|\geq c} (g(z)-f(z)) \, dz \label{eq: spreading ce 1} \\
    &= \inf_{|z| < c} H(z) \cdot (\P(|Z|< c) -\P(|Z'| < c)) - \sup_{|z| \geq  c} H(z) \cdot (\P(|Z'|\geq c) - \P(|Z|\geq c)) \nn \\
    &= (\inf_{|z| < c} H(z) - \sup_{|z| \geq  c} H(z))(\P(|Z|< c) - \P(|Z'|< c)) \label{eq: spreading ce 2} \\
    &\geq 0. \label{eq: spreading ce 3}
\end{align}
Equation~\eqref{eq: spreading ce 1} holds because $f(z)-g(z) \geq 0$ when $|z|< c$ and $f(z)-g(z)\leq 0$ when $|z|\geq c$. Note that this also implies that $\P(|Z|< c) - \P(|Z'|< c) = \int_{|z|< c} (f(z)-g(z))\, dz \geq 0$.
Equation~\eqref{eq: spreading ce 2} holds by substituting $\P(|Z|\geq c) = 1-\P(|Z|< c)$ and similarly for $Z'$. 
Equation~\eqref{eq: spreading ce 3} holds because $\inf_{|z|< c} H(z) \geq \sup_{|z|\geq c} H(z)$ by Lemma~\ref{thm: peaked ce}, and because $\P(|Z|< c) - \P(|Z'|< c) \geq 0$ as established previously. This completes the proof.
\end{proof}

\decepe*
\begin{proof}
By marginalizing \eqref{eq: cox epe} over $Y$, we have that
\[ \mathrm{EPE}(\b, R) = \E[H(\b^\T(X-X')) \: | \: X, X' \in R] \]
and similarly for $R'$. Note that since the expectation depends only on $X-X'$, it is translation invariant, so we may assume that $R = \prod_{i=1}^d [0, c_i]$ and $R' = \prod_{i=1}^d [0, c_i']$. We may further assume that $c_i = c'_i$ for all $i\geq2$ (i.e., $R$ and $R'$ differ in only a single side length): if we can show that the EPE decreases when only one side length is increased, then we can create a chain of at most $d$ inequalities in the EPE (where one side length increases at time) to prove the general inequality. Finally, since the expectation depends only on $\b^\T(X-X')$, we can replace $\b_1$ (the first coordinate of $\b$) with $c_1\b_1$ and $X_1, X'_1$ (the first coordinates of $X$ and $X'$) with $X/c_1, X'/c_1$. This effectively replaces $c_1$ with $c_1/c_1=1$ and $c'_1$ with $c'_1/c_1 = c > 1$. Thus, we may assume that $c_1=1$ and $c'_1 = c > 0$.

Let $Z = \b^\T (X-X')$ for $X,X'\sim\mathrm{Unif}(R)$ and $Z' = \b^\T (X-X')$ for $X,X'\sim\mathrm{Unif}(R')$, and let $f,g$ be the densities for $Z, Z'$ respectively. By Lemma~\ref{thm: spreading ce}, it suffices to show that there is a constant $c$ such that $f(z) \geq g(z)$ for $|z|< c$ and $f(z) \leq g(z)$ for $|z|\geq c$. Since the distributions of $Z$ and $-Z$ are equal (and similarly for $Z', -Z'$), $f$ and $g$ are both symmetric so it suffices to show this for $z\geq0$.

For any $c>0$, define $\psi_c(z) = \frac{c-|z|}{c^2}$. For a $d$-dimensional vector $v$, let $v_{\setminus 1}$ denote the $d-1$-dimensional vector consisting of all but the first entry of $v$. Finally, define $\varphi(z)$ to be the density of $\b_{\setminus 1}^\T (X_{\setminus1} - X'_{\setminus1})$ for $X_{\setminus1}, X'_{\setminus1}\sim \mathrm{Unif}(\prod_{i=2}^d [0, c_i])$.
The densities of $Z$ and $Z'$ are given by convolutions: $f = \psi_{|\b_1| c_1} * \varphi$ and $g = \psi_{|\b_1| c'_1} * \varphi$. We will now proceed show that the ratio $g(z)/f(z)$ is nondecreasing for $z\geq 0$ by showing that 
\[
\frac{d}{dz}\l[\frac{g(z)}{f(z)}\r] \leq 0 \quad \textrm{for} \quad z\geq 0.
\]
Since $g, f \geq 0$, this will suffice to prove the claim. (In particular, we can take the $c$ in Lemma~\ref{thm: spreading ce} to be $\sup \{z \: : \: g(z) / f(z) \leq 1\}$.) The derivative can be computed as follows:
\begin{equation} \label{eq: density ratio derivative}
    \frac{d}{dz}\l[\frac{g}{f}\r] = \frac{d}{dz}\l[\frac{\psi_{|\b_1| c'_1} * \varphi}{\psi_{|\b_1| c_1} * \varphi} \r]
    = \frac{(\psi_{|\b_1| c_1} * \varphi)(\psi_{|\b_1| c'_1}' * \varphi) - (\psi_{|\b_1| c'_1} * \varphi)(\psi_{|\b_1| c_1}' * \varphi)}{(\psi_{|\b_1| c_1} * \varphi)^2}.
\end{equation}
Since the denominator of \eqref{eq: density ratio derivative} is nonnegative, it suffices to show that the numerator is nonnegative. Note that $\frac{d}{dz}\psi_c(z) = -\mathrm{sgn}(z)/c^2$. Thus, the numerator of \eqref{eq: density ratio derivative} becomes
\begin{align}
    (\psi_{|\b_1| c_1} &* \varphi)(\psi_{|\b_1| c'_1}' * \varphi) - (\psi_{|\b_1| c'_1} * \varphi)(\psi_{|\b_1| c_1}' * \varphi) \nn \\[10pt]
    &= (\psi_{|\b_1| c_1} * \varphi)\l(\frac{-\mathrm{sgn}}{(|\b_1| c'_1)^2} * \varphi\r) - (\psi_{|\b_1| c'_1} * \varphi)\l(\frac{-\mathrm{sgn}}{(|\b_1| c_1)^2} * \varphi\r) \nn \\[10pt]
    &= \l(\l( \frac{\psi_{|\b_1|c_1}}{(|\b_1| c'_1)^2} - \frac{\psi_{|\b_1|c_1'}}{(|\b_1| c_1)^2} \r) * \varphi \r) (-\mathrm{sgn} * \varphi). \label{eq: density ratio derivative numerator}
\end{align}
A direct computation shows that
\[ \frac{\psi_{|\b_1|c_1}}{(|\b_1| c'_1)^2} - \frac{\psi_{|\b_1|c_1'}}{(|\b_1| c_1)^2} = \frac{|\b_1|c_1 - |z|}{(|\b_1|c_1)^2 (|\b_1|c_1')^2} - \frac{|\b_1|c_1' - |z|}{(|\b_1|c_1')^2 (|\b_1|c_1)^2} = \frac{|\b_1|(c_1-c_1')}{(|\b_1|c_1)^2 (|\b_1|c_1')^2}.\]
Since $c_1'\geq c_1$, this is a nonpositive constant function. Convolving against $\varphi$ (which is nonnegative function) will then result in a nonpositive constant.

Finally, we compute $-\mathrm{sgn}*\varphi$:
\begin{align}
    -(\mathrm{sgn}*\varphi)(z) &= -\int \varphi(y) \mathrm{sgn}(z-y) \, dy \nn \\
    &= \int_{y > z} \varphi(y) \, dy - \int_{y \leq z} \varphi(y) \, dy \nn \\
    &= \P(\b_{\setminus1}^\T (X_{\setminus 1} - X'_{\setminus 1}) > z) - \P(\b_{\setminus1}^\T (X_{\setminus 1} - X'_{\setminus 1}) \leq z). \nn
\end{align}
Since the distribution of $\b_{\setminus1}^\T (X_{\setminus 1} - X'_{\setminus 1})$ is symmetric about 0, it follows that 
\[\P(\b_{\setminus1}^\T (X_{\setminus 1} - X'_{\setminus 1}) > z) \leq 1/2, \quad \P(\b_{\setminus1}^\T (X_{\setminus 1} - X'_{\setminus 1}) \leq z) \geq 1/2\]
for $z\geq 0$. Thus, $(-\mathrm{sgn}*\varphi)(z) \leq 0$ for $z \geq 0$.

To conclude, we have now shown that \eqref{eq: density ratio derivative numerator} is the product of two nonpositive quantities, therefore it must be nonnegative. This completes the proof.
\end{proof}

\section{Convergence of the Conditional Rank Statistics}
In this section, we prove Theorem~\ref{thm: crs convergence} which states that the CRS converge uniformly to their expectation in probability. We restate the theorem here for convenience.
\crsconv*
\begin{proof}
By the Glivenko-Cantelli theorem, we have that $\sup_{t\in \R}|\hat{G}_n(t)-G(t)| \rightarrow 0$ almost surely, and therefore in probability as well.

Let $\cB^n_x = \{ \{x_i\}_{i=1}^n \: : \: \P_{\hat{G}_n}(\sup_{t\in\R} |\hat{G}_n(t)-G(t)|>\e \: | \: \{x_i\}_{i=1}^n) > \n\}$. We claim that $\P(\{x_i\}_{i=1}^n \in \cB_x^n) \rightarrow 0$ as $n\rightarrow\infty$ for any $\e, \n > 0$. Let $\P(\cB^n_x) = \g$. Again by Glivenko-Cantelli, for and $\e, \d>0$ we can choose $n$ large enough so that the following holds:
\begin{align}
    \d &> \P(\sup_t |\hat{G}_n(t) - G(t)|>\e) \nn\\
    &= \P(\cB^n_x) \P(\sup_t |\hat{G}_n(t) - G(t)|>\e \: | \: \cB^n_x) + (1-\P(\cB^n_x))\P(\sup_t |\hat{G}_n(t) - G(t)|>\e \: | \: \neg \cB^n_x) \nn \\
    &\geq \g \n + (1-\g) \cdot 0 = \g\n.
\end{align}
If we take $\n=\d^{1/2}$ then we see that $\g\leq\d^{1/2}$. Since $\d$ can be made arbitrarily small, $\n$ and $\g$ can also be made arbitrarily close to 0 for large enough $n$.

Choose $n$ large enough so that $\P(\{x_i\}_{i=1}^n \in \cB_x^n) \leq \d$, then consider a realization $\{x_i\}_{i=1}^n \not\in \cB_x^n$. By definition, we have 
\begin{equation} \label{eq: bad draw prob}
\P(\sup_{t\in\R}|\hat{G}_n(t)-G(t)|>\e \: | \: \{x_i\}_{i=1}^n) \leq \d.
\end{equation}

Next, consider $N$ draws of the survival times conditional on the $x_i$, let $t_i^{(j)}$ be the survival time for $x_i$ in the $j$-th draw, and let $\hat{G}_n^{(j)}$ be the associated empirical cdf for the survival times. Let 
\[
\cB_d = \{j\in [N] \: : \: \sup_{t\in\R}|\hat{G}_n^{(j)}(t) - G(t)| > \e\}
\]
be the indices of the ``bad'' draws. By \eqref{eq: bad draw prob}, $|\cB_d| \leq \d N + O(\sqrt{N})$ with high probability.

To simplify notation, let $T^* = T(x^*)$. We next claim the following: Given that $j\not\in \cB_d$, we have that $|\P_{T^*}(\hat{G}_n^{(j)}(T^*) \leq \a) - \P_{T^*}(G(T^*)\leq\a)| \leq 4C\e$ for all $\a$. To see this, observe that
\begin{equation} \label{eq: quantile prob diff}
|\P_{T^*}(\hfjn(T^*) \leq \a) - \P_{T^*}(G(T^*))| \leq \int_{t\geq 0} |\I\{\hfjn(t) \leq \a\} - \I\{G(t) \leq \a\}| \, dP(t), 
\end{equation}
where $dP(t)$ denotes the probability measure for $T^*$. Since $j \not\in \cB_d$, we have $\sup_t |\hfjn(t) - G(t)| \leq \e$. This implies the following:
\begin{itemize}
    \item If $\hfjn(t) \leq \a-\e$, then $G(t)\leq\a$ and vice versa.
    \item If $\hfjn(t) > \a+\e$, then $G(t)>\a$ and vice versa.
\end{itemize}
It follows that the integrand in \eqref{eq: quantile prob diff} is equal to 1 only if $\a-\e < \hfjn(t) \leq \a+\e$ or $\a-\e < G(t) \leq \a+\e$, and equal to 0 otherwise. Again since $j\not\in\cB_d$, observe that 
\[
\a-\e<\hfjn(t)\leq\a+\e \quad \Longrightarrow \quad \a-2\e<G(t)\leq\a+2\e.
\]
Thus, the integrand in \eqref{eq: quantile prob diff} is 1 only if $\a-2\e < G(t) \leq \a+2\e$. This implies that
\[
|\P_{T^*}(\hfjn(T^*) \leq \a) - \P_{T^*}(G(T^*))| \leq \P_{T^*}(\a-2\e<G(T^*)\leq\a+2\e).
\]
Observe that the set $\{t: \a-2\e < G(t) \leq \a+2\e\}$ has measure at most $4\e$ with respect to the marginal distribution of the core group survival times. Since $T^*$ is absolutely continuous with respect to the core group marginal survival time distribution and with bounded density, it follows that 
\[
\sup_\a|\P_{T^*}(\hat{G}_n(T^*) \leq \a) - \P_{T^*}(G(T^*) \leq \a)| \leq \sup_\a\P_{T^*}(\a-2\e<G(T^*)\leq\a+2\e)\leq 4C\e.
\]
In particular, this means that the set of indices $j$ for which $\sup_\a|\P_{T^*}(\hat{G}_n^{(j)}(T^*) \leq \a) - \P_{T^*}(G(T^*) \leq \a)| > 4C\e$ is a subset of $\cB_d$, so there are at most $\d N + O(\sqrt{N})$ such indices.

Let $\s\in S_n$ be a permutation. Let $\P(\cdot \: | \: \s)$ denote the probability of an event conditional on $T(x_{\s(1)}) \leq \ldots \leq T(x_{\s(n)})$, and define $\s_j$ to be the permutation such that $t_{\s_j(1)}^{(j)} \leq \ldots \leq t_{\s_j(n)}^{(j)}$. Note that the total number of possible orderings of the survival times is $n!$, which is independent of $N$, and furthermore each permutation of the survival times occurs with positive probability independent of $N$.

Stepping back for a moment, recall that $\P_{\hat{G}_n, T^*}(\hat{G}_n(T^*) \leq \a \, | \, \s)$ (the cdf of the conditional rank statistics) actually consists of $n+1$ jump discontinuities at $\a\in\{k/n\}_{k=0}^n$, corresponding to the $n+1$ possible ranks of $T^*$ among $T_{\s(i)}$. Since $n$ is fixed with respect to $N$, by the strong law of large numbers and a union bound, we will have that 
\[
\lim_{N\to\infty} \frac{\sum_{j=1}^N \P_{T^*}(\mathrm{rk}(T^*) = k \, | \, T^{(j)}_1, \ldots, T^{(j)}_n, \s) \I\{\s_j=\s\}}{\sum_{j=1}^N \I\{\s_j=\s\}} = \P_{\hat{G}_n, T^*}(\mathrm{rk}(T^*) = k \, | \, \s)
\]
almost surely, and simultaneously for all $k=1,\ldots,n+1$. Thus, for any $\a\in[0,1]$, we have
\begin{align}
\P_{\hat{G}_n, T^*}(\hat{G}_n(T^*) \leq \a \: | \: \s) &= \sum_{k=1}^{\lfloor \a n\rfloor + 1}\P_{\hat{G}_n, T^*}(\mathrm{rk}(T^*) = k \, | \, \s) \nn \\
&=\sum_{k=1}^{\lfloor \a n \rfloor + 1} \lim_{N\to\infty} \frac{\sum_{j=1}^N \P_{T^*}(\mathrm{rk}(T^*) = k \, | \, T^{(j)}_1, \ldots, T^{(j)}_n, \s) \I\{\s_j=\s\}}{\sum_{j=1}^N \I\{\s_j=\s\}} \nn \\
&=\lim_{N\to\infty} \frac{\sum_{j=1}^N \sum_{k=1}^{\lfloor \a n \rfloor + 1}\P_{T^*}(\mathrm{rk}(T^*) = k \, | \, T^{(j)}_1, \ldots, T^{(j)}_n, \s) \I\{\s_j=\s\}}{\sum_{j=1}^N \I\{\s_j=\s\}} \nn \\
&=\lim_{N \rightarrow \infty} \frac{\sum_{j=1}^N \P_{T^*}(\hat{G}_n^{(j)}(T^*) \leq \a) \I\{\s_j = \s\}}{\sum_{j=1}^N \I\{\s_j = \s\}}. \label{eq: large N estimate}
\end{align}

Let $\hat{p}_j = \hat{p}_j(\a) = \P(\hat{G}_n^{(j)}(T(x^*)) \leq \a)$ and $p^* = p^*(\a) = \P(F(T(x^*)) \leq \a)$. (We drop the $\a$ argument in what follows for clarity of presentation.) The preceding argument shows that $(\sum_{j=1}^N \hat{p}_j \I\{\s_j = \s\})/(\sum_{j=1}^N \I\{\s_j = \s\})$ converges to $\P_{\hat{G}_n, T^*}(\hat{G}_n(T^*) \leq \a \: | \: \s)$ uniformly in $\a$ as $N\rightarrow\infty$ and for all $\s$. It thus suffices to show that these finite sum estimates in \eqref{eq: large N estimate} are sufficiently close to $p^*$ (uniformly in $\a$) with high probability.

To this end, we want to know for which permutations $\s$ can equation \eqref{eq: large N estimate} deviate from $p^*$ by more than $4C\e+\n$ for some small $\n$ and for $N$ large enough. Let $N_\s = |\{j\: : \: \s_j = \s\}|$ and $N_\s^\times = |\{j\: : \: \s_j = \s \wedge \|\hat{p}_j-p^*\|_\infty>4C\e\}$. Observe that deterministically $\|\hat{p}_j - p^*\|_\infty \leq 1$, thus we have
\begin{align}
    \sup_\a &\l| \frac{\sum_{j=1}^N \hat{p}_j \I\{\s_j = \s\}}{\sum_{j=1}^N \I\{\s_j=\s\}} - p^* \r| \nn \\[10pt]
    &\leq \frac{\sum_{j \: : \: \|\hat{p}_j - p^*\|_\infty\leq 4C\e, \s_j = \s} \|\hat{p}_j - p^*\|_\infty + \sum_{j \: : \: \|\hat{p}_j - p^*\|_\infty > 4C\e, \s_j = \s} \|\hat{p}_j - p^*\|_\infty}{N_\s} \nn \\
    &\leq \frac{(N_\s - N_\s^\times)4C\e + N_\s^\times}{N_\s} \nn \\
    &\leq 4C\e + \frac{N_\s^\times}{N_\s}. \nn
\end{align}
In particular, we would need $4C\e + \n \leq 4C\e + N_\s^\times/N_\s$, which implies
\[ N_\s^\times \geq \n N_\s \geq \n \P(\s) N - O(\sqrt{N}) \]
with high probability. However, recall that the total number of indices $j$ with $|\hat{p}_j-p^*| > 4C\e$ is at most $\d N + O(\sqrt{N})$, so $\sum_\s N_\s^\times \leq \d N + O(\sqrt{N})$. But then we have
\begin{align}
    \d N + O(\sqrt{N}) &\geq \sum_\s N_\s^\times \nn \\
    &\geq \sum_{\s \in \cB_\s} N_\s^\times \nn \\
    &\geq \sum_{\s \in \cB_\s} \n \P(\s) N - O(\sqrt{N}) \nn.
\end{align}
Rearranging, we see that
\begin{equation} \label{eq: bad perm prob bound}
\sum_{\s \in \cB_\s} \P(\s) \leq \frac{\d}{\n} + O\l(\frac{1}{\sqrt{N}}\r).
\end{equation}
In particular, we can take $\n = \d^{1/2}$.

To recap, we have shown that for any $\e, \d > 0$, we can choose $n$ large enough such that $\P(\{x_i\}_{i=1}^n \in \cB_x^n) < \d$, and such that $\P(\s \in \cB_\s \: | \: \{x_i\}\not\in\cB_x^n) \leq 2\d^{1/2}$. (The factor of 2 is to absorb the $O(1/\sqrt{N})$ term in \eqref{eq: bad perm prob bound}.) It follows that
\begin{align}
\P&\l\{\sup_\a \l|\P(\hat{G}_n(T(x^*))\leq\a \: | \: \s) - \P(G(T^*)) \leq \a)\r| > 4C\e + \d^{1/2} \r\} \nn \\
&\leq \P(\{x_i\}_{i=1}^n \in \cB_x^n) + \P(\s \in \cB_\s \: | \: \{x_i\}_{i=1}^n \not\in \cB_x^n) \nn \\
&\leq \d + 2\d^{1/2}. \nn
\end{align}
As both the size of the error $4C\e + \d^{1/2}$ and the failure probability $\d + 2\d^{1/2}$ can be made arbitrarily close to 0 by letting $\e, \d \rightarrow 0$, we have that $\P(\hat{G}_n(T(x^*)) \leq \a | \s) \stackrel{p}{\rightarrow} \P(G(T(x^*)) \leq \a)$ uniformly in $\a$, as desired.
\end{proof}

\section{Proof of Theorem~\ref{thm: main}} \label{appendix: proofs}

\subsection{Assumptions and Notation}
We first restate our assumptions for convenience.

\begin{enumerate}[label=A\arabic*.]
    \item The hazard function for the entire dataset has the form $\lambda(t; x) = \lambda_0(t)e^{h(x)}$ for some unknown risk function $h$.
    \item There is no censoring in the data.
    \item There is a unique largest region $R^*$ which minimizes the EPE, and $R^*$ is an axis-aligned box.
    Conditional on $x\in R^*$, we have $h(x) = \b^\T x$ for some $\b$, i.e., the Cox model is well-specified.
    \item The core group selection procedure finds a group of points which belong to $R^*$, and the Cox model fit to these points recovers the true parameters $\b$. 
    \item The error between the finite conditional rank statistics and its large-sample limit according to Theorem~\ref{thm: crs convergence} is negligible. We use this limiting distribution for the analysis, rather than the finite sample version described in Section~\ref{sec: crs}.
\end{enumerate}

We define the following quantities:
\begin{itemize}
    \item 
    $T(x)$ denotes the survival time for a point with features $x$, i.e., for each datapoint $x_i$ in the dataset, $t_i = T(x_i)$. By our assumptions on the data generating distribution, $T(x)$ is a survival time with hazard function $\lambda(t, x) = \lambda_0(t)e^{h(x)}$.
    \item $G(t; R_{\mathrm{core}})$ denotes the marginal CDF for survival times sampled from points belonging to the core group. That is, $G(t) = \P(T(X) \leq t \: | \: X \in R_{\mathrm{core}})$, where the probability is computed with respect to both the randomness in $X \in R_{\mathrm{core}}$ and $T(X)$. To simplify the notation, we will typically write $G(t) = G(t; R_{\mathrm{core}})$.
    \item $G^{-1}: [0,1) \rightarrow \R_{\geq 0}$ denotes the inverse CDF.
    \item $\pI$ denotes the type I error rate, i.e., the probability that a point which belongs to $R^*$ is rejected.
    \item $\pII$ denotes the type II error rate, i.e., the probability that we fail to reject a point close to each face of $R^*$.
    \item $2n$ denotes the total number of datapoints in the training dataset, $n$ of which are used to find the core group and $n$ of which are used for the rejection/expansion step. We split the dataset to avoid creating dependencies between the datapoints after the core group selection step.
\end{itemize}
The necessity and validity of these assumptions is discussed in Appendix~\ref{appendix: assumptions}.

\subsection{Theoretical Implementation of DDGroup}
We use a modified version of the algorithm for the proof of Theorem~\ref{thm: main}. Several of these descriptions are replicated (nearly) verbatim from \cite{izzo23subgroup} for the reader's convenience.

Let $U \subseteq \R^d$. We define the \emph{directed infinity norm} $\|x\|_{U,\infty}$ by
$$ \|x\|_{U,\infty} = \max_{u \in U} \: x^\T u. $$
We note that for many sets $U$, $\|\cdot\|_{U,\infty}$ may not be a norm, nor even a seminorm. In what follows, $U$ will initially be defined as $U = \{\pm s^\pm_j e_i\}_{i=1}^d$. When $s^\pm_j = 1$ for all dimensions $j$ and signs $\pm$,  $\|\cdot\|_{U,\infty} = \|\cdot\|_\infty$ coincides with the usual infinity norm on $\R^d$. With generic ``speed'' parameters $s^\pm_j$, this quantity corresponds to the $\ell_\infty$ norm after first rescaling the positive/negative axes by $1/s^+_j$ and $1/s^-_j$, respectively.

\begin{algorithm}[h!]
\caption{\textsc{CoreGroup}$(S_{\mathrm{core}}, D)$} \label{alg: core group}
\begin{algorithmic}
\REQUIRE Core group shape $S_{\mathrm{core}}$, dataset $D = \{(x_i, t_i)\}_{i=1}^n$
\STATE $\textrm{EPE}^* \gets \infty$
\FOR{$(x, t) \in D$}
    \STATE $D_\textrm{nbhd} \gets \{(x', t') \in D \: | \: x'-x \in S_{\mathrm{core}}\}$ \COMMENT{$D_\textrm{nbhd}$ contains all points in a set of shape $S_{\mathrm{core}}$ centered at $x$. For instance, $S_{\mathrm{core}}$ could be an $\ell_\infty$ ball of radius $r$; then $D_\textrm{nbhd}$ contains all points within an $\ell_\infty$-distance $r$ from $x$.}
    \STATE $\hat{\beta} \gets \textsc{FitCox}(D_\textrm{nbhd})$ \COMMENT{Fit the Cox model using the log partial likelihood.}
    \IF{$\textsc{EPE}(\hat{\b}, D_\textrm{nbhd}) < \textrm{EPE}^*$}
        \STATE $D_\textrm{core} \gets D_\textrm{nbhd}$
        \STATE $R_\textrm{core} \gets x + S_{\mathrm{core}}$
        \STATE $\textrm{EPE}^* \gets \textsc{EPE}(\hat{\b}, D_\textrm{nbhd})$
    \ENDIF
\ENDFOR
\RETURN $R_{\textrm{core}}, D_\textrm{core}$
\end{algorithmic}
\end{algorithm}

\begin{algorithm}[h!]
\caption{\textsc{GrowBox}$(\bar{x}, X_{\textrm{rej}}, \{s^\pm_j\}_{j=1}^d$)} \label{alg: growing box}
\begin{algorithmic}
\REQUIRE Starting point (center) $\bar{x}$, rejected points $X_{\textrm{rej}}$, side expansion speeds $\{s^\pm_j\}_{j=1}^d$
\STATE $\hat{R} \gets \emptyset$
\STATE $U \gets \{ -s^-_j e_j, \, s^+_j e_j\}_{j=1}^d$ \COMMENT{$e_j$ denotes the $j$-th standard basis vector.}
\STATE $X_\textrm{rej} \gets X_\textrm{rej} + \{-\bar{x}\}$ \COMMENT{Center the points at $\bar{x}$. $+$ denotes Minkowski sum.}
\WHILE{$X_{\textrm{rej}} \neq \emptyset$}
    \STATE $x^* \gets \argmin_{x\in X_{\textrm{rej}}} \{ \|x\|_{U, \infty} \}$
    \STATE $a^* \gets \|x^*\|_{U, \infty}$
    \STATE $u^* \gets \argmax_{u \in U} \{ u^\T x^* \}$ \COMMENT{$u^*$ is the next support direction for the polytope}
    \STATE Add $(u^*, a^*)$ to $\hat{R}$
    \STATE Remove $u^*$ from $U$
    \STATE $X_{\textrm{rej}} \gets \{x \in X_{\textrm{rej}} \: | \: x^\T u^* < a^* \}$
\ENDWHILE
\RETURN $\hat{R} + \{\bar{x}\}$ \COMMENT{Undo the centering procedure from the first part of the algorithm.}
\end{algorithmic}
\end{algorithm}

\begin{algorithm}[h!]
\caption{\textsc{DDGroup}$(R_{\mathrm{core}}, \{s^\pm_j\}_{j=1}^d, D)$} \label{alg: pipeline}
\begin{algorithmic}
\REQUIRE Core group shape $S_{\mathrm{core}}$, growth speeds $s_j^\pm$, dataset $D=\{(x_i, t_i)\}_{i=1}^{2n}$
\STATE
\STATE \textbf{Phase 1:} Find a core group and fit a coarse model using half of the data.
\STATE $R_\textrm{core}, D_\textrm{core} \gets \textsc{CoreGroup}(S_{\mathrm{core}}, \{x_i\}_{i=n+1}^n)$
\STATE $\hat{\b} \gets \textsc{FitCox}(D_\textrm{core})$ \COMMENT{This fits the Cox model by minimizing the log partial likelihood. We assume $\hat{\b}=\b$ (Assumption~\ref{assm: core group})}
\STATE

\STATE \textbf{Phase 2:} Label which points should be excluded using the other half of the training data.
\FOR{$i = 1, \ldots, n$}
    \STATE $\ell_i \gets \I\{G(t_i; R_{\mathrm{core}}) \not\in [\uq(x), \, \oq(x)]\}$ \COMMENT{The lower and upper quantiles $\uq,\oq$ are defined in Lemma~\ref{thm: rej quantiles}.}
\ENDFOR
\STATE $X_\textrm{rej} \gets \{x_i \in X \: | \: \ell_i = 1\}$
\STATE

\STATE \textbf{Phase 3:} Approximate $R^*$.
\STATE $\bar{x} \gets \textsc{Mean}(\{x \: | \: (x, t) \in D_\textrm{core}\})$
\STATE $\hat{R} \gets \textsc{GrowBox}(\bar{x}, X_\textrm{rej}, \{s_j^\pm\}_{j=1}^d)$
\RETURN $\hat{R}$
\end{algorithmic}
\end{algorithm}

\subsection{Proof}
We now proceed with the proof of Theorem~\ref{thm: main}, which is broken down into five steps.

\begin{lem} \label{thm: time warp}
We may assume WLOG that $\lambda_0(t) \equiv 1$.
\end{lem} 
\begin{proof}
Consider the random variable $\tilde{T} = \int_0^T \lambda_0(s)\, ds$. Note that this is a monotonic change of variables since $\lambda_0(s) > 0$. In addition, note that the survival function $\tilde{S}(t,x)$ of $\tilde{T}$ conditional on features $X=x$ is given by
\begin{align}
\tilde{S}(t,x) &= \P(\tilde{T} \geq t \: | \: X=x) \nn \\
&= \P\l(\int_0^T \lambda_0(s)\, ds \geq t \: | \: X=x\r) \nn \\
&= \P(T \geq \Lambda_0^{-1}(t) \: | \: X=x) \nn \\
&= \exp\l( -e^{h(x)} \Lambda_0(\Lambda_0^{-1}(t)) \r) \nn \\
&= \exp(-e^{h(x)}t). \label{eq: time warp pf}
\end{align}
This implies that $\tilde{T}$ has hazard function $\tilde{\lambda}(t;x) = e^{h(x)}$, which in particular means that the baseline hazard function under this transformation is $\tilde{\lambda}_0(t) \equiv 1$. Since the transformation is monotonic, all of the ranks will be preserved, so all of the results which hold for $T$ hold also for $\tilde{T}$ and vice-versa.
\end{proof}

\begin{lem} \label{thm: upper bound overall fpr}
Let $p$ be an upper bound on the probability that a point $x\in R^*$ is rejected. If $p \leq \pI/n$, then the probability that \emph{any} point in $R^*$ is rejected is at most $\pI$.
\end{lem}
\begin{proof}
Since the size of the dataset is $n$, there are clearly at most $n$ points in $R^*$. The result then follows from a simple union bound over these points.
\end{proof}

\begin{lem} \label{thm: rej quantiles}
Define the lower and upper rejection quantiles
\[
\uq(x) = 1-\E_{X\sim \mathrm{Core}}\l[\l(1-\pt\r)^{e^{\b^\T(X-x)}}\r], \hspace{.25in}
\oq(x) = 1-\E_{X\sim \mathrm{Core}}\l[\l(\pt\r)^{e^{\b^\T(X-x)}}\r]
\]
respectively. Using these quantiles, the probability of incorrectly rejecting a point $x\in R^*$ is at most $p$, i.e., 
\[\P(G(T(x)) \not\in [\uq(x), \oq(x)]) \leq p.\]
\end{lem}
\begin{proof}
By a union bound, it suffices to choose $\underline{q}(x)$, $\overline{q}(x)$ such that
\[
\P(G(T(x)) < \underline{q}(x)) \leq p/2, \hspace{.25in} \P(G(T(x)) > \overline{q}(x)) \leq p/2.
\]

Let us start with the lower quantile. We will abbreviate $\underline{q}=\underline{q}(x)$. Given $x$, we wish to determine the maximum possible $\underline{q}$ such that
\[
\P(T(x) \leq G^{-1}(\uq)) \leq \frac{p}{2} \quad \Leftrightarrow \quad \P(T(x) > G^{-1}(\uq)) \geq 1-\pt.
\]
From Lemma~\ref{thm: time warp}, it suffices to consider the case where $\lambda_0(t) \equiv 1$. In this case, since $x\in R^*$, $T(x)$ is an exponential random variable with rate $e^{\b^\T x}$ and we can compute the tail probability explicitly. In particular, we have
\begin{align}
\phantom{\Leftrightarrow}\quad &\P(T(x) > G^{-1}(\uq)) = \exp(-e^{\b^\T x}G^{-1}(\uq)) \geq 1-\pt \\[10pt]
\Leftrightarrow \quad &-e^{\b^\T x} G^{-1}(\uq) \geq \log(1-\pt) \\[10pt]
\Leftrightarrow \quad &G^{-1}(\uq) \leq -\log(1-\pt)e^{-\b^\T x} \\[10pt]
\Leftrightarrow \quad &\uq \leq G\l( -\log(1-\pt)e^{-\b^\T x} \r) \\[10pt]
&\phantom{\uq} = \P_{T\sim \mathrm{Core}}\l(T \leq -\log(1-\pt)e^{-\b^\T x}\r) \\[10pt]
&\phantom{\uq} = \E_{X\sim \mathrm{Core}}\l[ 1 - \exp\l( -e^{\b^\T X} \cdot (-\log(1-\pt)e^{-\b^\T x}) \r) \r] \\[10pt]
&\phantom{\uq} = 1 - \E_{X\sim \mathrm{Core}}\l[(1-\pt)^{e^{\b^\T(X-x)}}\r]. \label{eq: lower rej quantile}
\end{align}
In particular, we can take $\uq(x)$ equal to the expression in \eqref{eq: lower rej quantile}. A similar calculation for the upper quantile yields the expression for $\oq(x)$.
\end{proof}

\begin{lem} \label{thm: lower bound overall fnr}
Let $\{\tilde{x}_i\}_{i=1}^m$ be $m$ points with independent survival times be such that 
\[
\P(\tilde{x}_i \textrm{ is rejected}) \geq \n
\]
for all $i$. If $\n \geq 1 - \pII^{1/m}$, then the probability that none of the $x_i$ are rejected is at most $\pII$.
\end{lem}
\begin{proof}
By the independence of the survival times, the probability that none of the points are rejected is at most
\[
(1-\n)^m \leq (1-(1-\pII^{1/m}))^m = \pII
\]
as desired.
\end{proof}

\begin{lem} \label{thm: points near boundary}
Suppose $\{x_i\}_{i=1}^n$ are drawn i.i.d. from the uniform distribution on $\mathcal{B}\subseteq \R^d$ with $\mathrm{vol}(\mathcal{B}) = B < \infty$. Let $S \subseteq \mathcal{B}$ be any region such that $\mathrm{vol}(S) \geq C > 0$. Then $|\{x_i \: | \: x_i \in S\} \geq \frac{Cn}{2B}$ with probability at least $1-1/n$ for large enough $n$.
\end{lem}
\begin{proof}
Since the $x_i$ are drawn from the uniform distribution, we have that $\P(x_i \in S) \geq C/B$. This means that $\I\{x_i \in S\} \sim \mathrm{Bernoulli}(q)$ with $q \geq C/B$. Thus by Hoeffding's inequality, with probability at least $1-1/n$ we have that
\[
|\{x \: : \: |u^\T x| < r\}| \geq qn - \sqrt{\frac{n\log n}{2}} \geq \frac{Cn}{2B}
\]
for large enough $n$.
\end{proof}

\begin{lem} \label{thm: lower bound good rej}
There exists $C_\e = O(\log \e^{-1})$ such that if $|h(x) - \b^\T x| \geq C_\e$, then the probability of rejecting $x$ is at least $\n$:
\[
\P(G(T(x)) \not\in [\underline{q}(x), \, \overline{q}(x)]) \geq \n.
\]
\end{lem}
\begin{proof}
By Lemma~\ref{thm: time warp}, it suffices to consider $\lambda_0(t)\equiv 1$ so that $T(x)$ is exponential with rate $e^{h(x)}$. We will use this fact throughout the proof. We will also make repeated use of the fact that $G$ and $G^{-1}$ are monotonically increasing functions without stating this explicitly.

By Lemma~\ref{thm: c epsilon value}, we have that $-\log^2(\frac{1}{1-p/2}) + \log^2(\frac{1}{1-\n})$ and $\log^2(2/p) -\log^2(1/\n)$ are both $O(\log \e^{-1})$. Thus, it suffices to show that if
\[
h(x) \leq \b^\T x - \log^2\frac2p +\log^2\frac1\n \quad \textrm{or} \quad h(x) \geq \b^\T x -\log^2(\frac{1}{1-p/2}) + \log^2(\frac{1}{1-\n})
\]
then the rejection probability is at least $\n$ for $x$. The value of $C_\e$ can then be chosen as the $O(\log\e^{-1})$ upper bound for these two quantities.

We will now consider two cases based on whether $h(x) > \b^\T x$ or $h(x) < \b^\T x$. In the former case, we use the fact that 
\[
\P(G(T(x)) \not\in [\underline{q}(x), \, \overline{q}(x)]) \geq \P(G(T(x)) < \underline{q}(x))
\]
and lower bound the RHS; similarly, when $h(x) < \b^\T x$, we use the fact that 
\[
\P(G(T(x)) \not\in [\underline{q}(x), \, \overline{q}(x)]) \geq \P(G(T(x)) > \overline{q}(x))
\]
and lower bound the RHS.

Let us first suppose $h(x) > \b^\T x$. Let $\uq=\uq(x)$. Then we have
\begin{align}
\P(G(T(x)) \leq \uq) \geq \n \quad \Leftrightarrow \quad &\P(T(x) > G^{-1}(\uq)) \leq 1-\n \\[10pt]
\Leftrightarrow \quad &\exp(-e^{h(x)}G^{-1}(\uq)) \leq 1-\n \\[10pt]
\Leftrightarrow \quad &e^{h(x)}G^{-1}(\uq) \geq -\log(1-\n) \\[10pt]
\Leftrightarrow \quad &h(x) \geq \log^2\l(\frac{1}{1-\n} \r) - \log G^{-1}(\uq). \label{eq: hazard lb 1}
\end{align}

Using the expression for $\uq=\uq(x)$ from Step 3, for any $t \geq 0$, we have
\begin{align}
&\P_{T\sim\mathrm{Core}}(T \leq t) \leq \uq \\
\Leftrightarrow \quad &\E_{X\sim\mathrm{Core}}\l[1-\exp(-e^{\b^\T X}t)\r] \leq 1-\E_{X\sim\mathrm{Core}}\l[(1-\pt)^{e^{\b^\T(X-x)}}\r] \\[10pt]
\Leftrightarrow \quad &\E_{X\sim\mathrm{Core}}\l[\exp\l(\log(1-\pt) e^{-\b^\T x} e^{\b^\T X}\r)\r] \leq \E_{X\sim\mathrm{Core}}\l[\exp(-te^{\b^\T X})\r]. \label{eq: pointwise ineq in expectation}
\end{align}
As long as $t \leq -\log(1-\pt)e^{-\b^\T x}$, the integrand on the LHS of \eqref{eq: pointwise ineq in expectation} is pointwise less than or equal to the integrand on the RHS. In particular, with $t$ equal to this upper bound we have $G(t) = \P_{T\sim\mathrm{Core}}(T \leq t) \leq \uq$, which implies that 
\[ G^{-1}(\uq) \geq -\log(1-\pt)e^{-\b^\T x} \quad \Longrightarrow \quad -\log G^{-1}(\uq) \leq \b^\T x - \log^2\l(\frac{1}{1-p/2}\r). \]
Thus we have
\begin{equation}
    h(x) \geq \b^\T x + \log^2\l(\frac{1}{1-\n} \r) - \log^2\l(\frac{1}{1-p/2}\r) \geq \log^2\l(\frac{1}{1-\n} \r) - \log G^{-1}(\uq).
\end{equation}
In particular, inequality~\eqref{eq: hazard lb 1} is satisfied, which implies that $\P(G(T(x)) \leq \uq(x)) \geq \n$ as desired.

Otherwise, suppose that $h(x) < \b^\T x$. Now we have
\begin{align}
\P(G(T(x)) > \oq) \geq \n \quad \Leftrightarrow \quad &\exp(-e^{h(x)}G^{-1}(\oq)) \geq \n \\[10pt]
\Leftrightarrow \quad &e^{h(x)}G^{-1}(\oq) \leq -\log\n \\[10pt]
\Leftrightarrow \quad &h(x) \leq \log^2\l(\frac{1}{\n} \r) - \log G^{-1}(\oq). \label{eq: hazard lb 2}
\end{align}
In order to lower bound $-\log G^{-1}(\oq)$, we wish to establish conditions on $t$ such that $G(t) = \P_{T\sim\mathrm{Core}}(T \leq t) \geq \oq$. Using the expression from ..., we have
\begin{align}
&\P_{T\sim\mathrm{Core}}(T \leq t) \geq \oq \\
\Leftrightarrow \quad &\E_{X\sim\mathrm{Core}}\l[1-\exp(-e^{\b^\T X}t)\r] \geq 1-\E_{X\sim\mathrm{Core}}\l[(\pt)^{e^{\b^\T(X-x)}}\r] \\[10pt]
\Leftrightarrow \quad &\E_{X\sim\mathrm{Core}}\l[\exp\l(\log(\pt)e^{-\b^\T x} e^{\b^\T X}\r)\r] \geq \E_{X\sim\mathrm{Core}}\l[\exp(-te^{\b^\T X})\r]. \label{eq: pointwise ineq in expectation 2}
\end{align}
Again, inequality~\eqref{eq: pointwise ineq in expectation 2} can be made to hold pointwise in the integrand provided that $t \geq -\log(\pt)e^{-\b^\T x}$. It therefore follows that
\[
G^{-1}(\oq) \leq -\log(\pt)e^{-\b^\T x} \quad \Longrightarrow \quad -\log G^{-1}(\uq) \geq \b^\T x - \log^2\l(\frac{2}{p}\r)
\]
Finally, we have
\[
h(x) \leq \b^\T x - \log^2\l(\frac{2}{p}\r) + \log^2\l(\frac1\n\r) \leq \log^2\l(\frac1\n\r) - \log G^{-1}(\oq).
\]
Thus \eqref{eq: hazard lb 2} holds and therefore $\P(G(T(x)) \geq \oq(x)) \geq \n$ as desired.
\end{proof}

To prove the desired result, we can now directly apply logic from the analogous proof in \cite{izzo23subgroup}. Specifically, since we have assumed that the core group lies within $R^*$ and we have shown that no points in $R^*$ will be rejected (Step 2), the logic of \cite{izzo23subgroup} shows that $R^* \subseteq \hat{R}_N$ given correct settings for the expansion speed $s_j^\pm$ of each side of the box. Similarly, the choice of $C_\e$ implies that there will be a rejected point within an $O(\e)$ distance from each face of $R^*$ (Steps 4 \& 5), meaning that each face of $\hat{R}_N$ will stop expanding within $O(\e)$ distance of the corresponding face of $R^*$ and yielding $\mathrm{vol}(\hat{R}_N \setminus R^*) = O(\e)$. This completes the proof.

\begin{lem} \label{thm: c epsilon value}
Let $p=\pI/n$, $\n=1-\pII^{1/m}$, and $m\geq c\e n$. Then as $n\rightarrow\infty$, we have
\[
- \log^2\l(\frac{1}{1-p/2}\r) + \log^2\l(\frac{1}{1-\n}\r)  = O(\log \e^{-1}),
\]
\[
\log^2\l(\frac2p \r) - \log^2\l(\frac1\n\r) = o(1) = O(\log \e^{-1}).
\]
\end{lem}
\begin{proof}
We begin with the first bound. Substituting for $p$, we have that
\[
-\log \log \frac{1}{1-\pI/2n} \leq -\log \log (1+\pI/2n) \leq -\log \frac{\pI}{4n}
\]
for large enough $n$. (Here we have used $\log(1+z) \leq z/2$ for small $z>0$.)
Substituting for $\n$ and $m$, we also have
\[
\log\log\frac{1}{1-\n} \leq \log \l(\frac1m \log\pII^{-1}\r) \leq -\log(c\e n) + \log^2 \pII^{-1}.
\]
It therefore follows that
\begin{align*}
- \log^2\l(\frac{1}{1-p/2}\r) + \log^2\l(\frac{1}{1-\n}\r) &\leq \log \frac{4n}{\pI} - \log(c\e n) + \log^2 \pII^{-1} \\
&= \log\e^{-1} + \log(4c\pI^{-1}) + \log^2\pII^{-1} \\
&= O(\log \e^{-1}).
\end{align*}

We now turn to the second bound. First, observe that $\pII^{1/m} = e^{\frac1m \log\pII} \geq 1 + \frac1m \log \pII$. We therefore have that
\[
-\log^2\l(\frac1\n\r) \leq -\log^2\frac{1}{1-\pII^{1/m}} \leq -\log^2\frac{m}{\log\pII^{-1}} \leq  -\log^2\frac{c\e n}{\log\pII^{-1}}.
\]
Thus, it follows that
\[
\log^2\l(\frac2p \r) - \log^2\l(\frac1\n\r) \leq \log^2\frac{2n}{\pI} - \log^2\frac{c\e n}{\log\pII^{-1}}.
\]
We can therefore apply Lemma~\ref{thm: difference of loglog} with $c_1 = 2/\pI$ and $c_2 = c\e/\log(\pII^{-1})$ to conclude that $\log^2\l(\frac2p \r) - \log^2\l(\frac1\n\r)\to 0$. In particular, this term is also $O(\log\e^{-1})$.
\end{proof}

\begin{lem} \label{thm: difference of loglog}
For any positive constants $c_1, c_2 > 0$, we have $\lim_{n\rightarrow\infty} \log^2(c_1n) - \log^2(c_2n) = 0$.
\end{lem}
\begin{proof}
This is a straightforward calculation:
\[
\log^2(c_1n) - \log^2(c_2n) = \log\l(\frac{\log(c_1n)}{\log(c_2n)} \r) =  \log \l(\frac{\log c_1 +\log n}{\log c_2 + \log n} \r) =  \log \l(1 + \frac{\log c_1 - \log c_2}{\log c_2 + \log n} \r) \rightarrow \log(1+0) = 0.
\]
Note that we have used the continuity of the logarithm.
\end{proof}

\main*
\begin{proof}
Let $p = \pI/n$. By Lemma~\ref{thm: rej quantiles}, the probability that any point $x\in R^*$ is rejected is at most $p$. Since there are at most $n$ points in $R^*$ (as the total size of the dataset is $n$), by Lemma~\ref{thm: upper bound overall fpr}, the probability that any point in $R^*$ is rejected is at most $\pI$.

Since $R^*$ is an axis-aligned box, we can write $R^* = \prod_{j=1}^d [a_j, b_j]$. For any vector $x \in \R^d$, let $x[k]$ denote the $k$-th entry of $x$. As in \cite{izzo23subgroup}, for each $j=1,\ldots, d$, we define
\[
\partial R^*_{\e, j, +} = \{ x\in \R^d \: | \: b_j \leq x[j] \leq b_j + \e, \: a_k \leq x[k] \leq b_k, k \neq j\},
\]
\[
\partial R^*_{\e, j, -} = \{ x\in \R^d \: | \: a_j - \e \leq x[j] \leq a_j, \: a_k \leq x[k] \leq b_k, k \neq j\}.
\]
These are the sets of points which are at most $\e$ ``above'' the upper dimension $j$ face of $R^*$ and ``below'' the lower dimension $j$ face of $R^*$, respectively.

Fix a dimension $j$ and let $m$ be the number of points in $\partial R^*_{\e, j, +}$. By Lemma~\ref{thm: points near boundary}, $m\geq c\e n$ for some constant $c > 0$ with probability at least $1-1/n$. Let $\n = 1-\pII^{1/m}$, and let $C_\e = O(\log\e^{-1})$ be as defined in Lemma~\ref{thm: c epsilon value}. Then by the result of Lemma~\ref{thm: c epsilon value}, we have that each of the $m$ points in $\partial R^*_{\e, j, +}$ has probability at least $\n$ of being (correctly) rejected. By Lemma~\ref{thm: lower bound overall fnr}, the probability that no points in $\partial R^*_{\e, j, +}$ are rejected is at most $\pII$. The same logic clearly holds for $R^*_{\e, j, -}$ as well. By a union bound over $j=1,\ldots, d$ and over the sign $\pm$, we have that all of the $\partial R^*_{\e, j, \pm}$ contain a rejected point with probability at least $1-2d\cdot (\pII+1/n)$.

By Assumption~\ref{assm: core group}, the core group is a subset of $R^*$, hence its average $\bar{x}$ is contained in $R^*$.

For the remainder of the proof, we will condition on the event that $R^*$ contains no rejected points and all of the $\partial R^*_{\e, j, \pm}$ contain at least one rejected point. Again by a union bound, this occurs with probability at least
\[
1 - \pI - 2d\cdot (\pII + 1/n).
\]

We can now follow the logic of Theorem~4.3 of \cite{izzo23subgroup} with almost no modification, which we replicate here for convenience. Let $s_j^\pm = d(\bar{x}, \partial R^*_{j, \pm})$ be the distance from the core group center to the appropriate face of $R^*$. Note that Algorithm~\ref{alg: growing box} with these speeds and this center is equivalent to running the algorithm from the origin and with uniform speeds, after shifting the data so that $\bar{x}$ lies at the origin and then rescaling each positive and negative axis by $1/s_j^+$ and $1/s_j^-$, respectively. In this case, $R^*$ is transformed into a $\ell_\infty$ ball of radius 1 centered at the origin. We have also conditioned on the event that $R^*$ contains no rejected points, and the transformations we performed above preserve this fact. Since the region returned by Algorithm~\ref{alg: growing box} returns a region which contains the largest centered $\ell_\infty$ ball with no rejected points in it, and $R^*$ is a centered $\ell_\infty$ ball with no rejected points, we must have $R^* \subseteq \hat{R}$ as desired.

Because each $\partial R^*_{\e, j, \pm}$ contains at least one rejected point, Algorithm~\ref{alg: growing box} will stop growing the $(j,\pm)$ side of $\hat{R}$ at some point in $\partial R^*_{\e, j, \pm}$. In particular, this means that if the final region $\hat{R}_n = \prod_{j=1}^d [\ell_j, u_j]$ then we have 
\[
\ell_j \geq a_j - \e, \hspace{.25in} u_j \leq b_j + \e
\]
for all $j$. (If any of these inequalities were violated, then since $R^* \subseteq \hat{R}_n$, by convexity $\hat{R}_n$ would also have to contain a rejected point.) 
Thus we can conclude that $\mathrm{vol}(R^*\setminus \hat{R}_n) = O(\e)$. This completes the proof.
\end{proof}

\subsection{Discussion of the Result} \label{appendix: main thm discussion}
Regarding the convergence result itself, the natural point of comparison is with the convergence result of the version of DDGroup intended for use with linear regression in \cite{izzo23subgroup}. At first glance, the result of Theorem~\ref{thm: main} may appear weaker: given a fixed effect size ($C_\e$) and sufficiently many samples, we recover $R^*$ up to a fixed $O(\e)$ error in volume. In contrast, \cite{izzo23subgroup} obtains arbitrary precision given enough samples and at a fixed effect size. The reason for this disparity is that \cite{izzo23subgroup} assumed a lower bound on the \emph{variance} of $y|x$ outside of $R^*$; our assumption on the effect size is on the equivalent of the regression function. Since detections are performed based on the deviation of a single point, and the Gaussian fluctuations in \cite{izzo23subgroup} grow as $\s \log n$, where $\s$ is the standard deviation and $n$ the number of datapoints, their convergence result would also require a effect size in the regression function which increases with the dataset size $n$ (at least $\Omega(\log n)$) in order to obtain arbitrary precision with increasing sample size. Indeed, the proof of Theorem~\ref{thm: main} shows that $C_\e = O(\log\e^{-1})$, so to obtain an error of size $O(n^{-c})$ for some constant $c>0$ (as is the case in \cite{izzo23subgroup}), DDGroup adapted to the Cox model would also need an effect size of order $O(\log n)$. Thus, these two results are comparable in their recovery guarantees.

We also remark that the bounds on the effect size resulting from Lemma~\ref{thm: lower bound good rej} are not symmetric for the upper and lower tails. The effect size is determined by the amount which the parameter of an exponential distribution must be changed in order to move its upper or lower tails to lie above or below (respectively) the upper and lower quantiles $\uq(x)$ and $\oq(x)$. Although $\uq$ and $\oq$ are defined as symmetric upper/lower quantiles of the marginal core group survival distribution, the size of change required to move the upper and lower tails of the exponential distribution past a certain point is not symmetric. Thus, we should not expect the effect size required for upper/lower rejections to be the same.

Because of the asymmetry in the upper and lower tails, it is true that we can improve (reduce) the required effect size by balancing the thresholds required for upper and lower tail rejections. It is possible it will even lead to asymptotic improvement with respect to $\e^{-1}$, but since the point of the theory is just to get a sense of the algorithm's performance, we do not explore this further here.

\section{Discussion of Assumptions} \label{appendix: assumptions}
We briefly discuss the necessity and validity of the assumptions.

Relaxing Assumptions~\ref{assm: hazard func} (to include general hazard functions outside of $R^*$ and \ref{assm: censoring} (to include censoring) should be possible. As the present paper is the first to consider even the more basic forms of these questions, we leave these more challenging extensions to future work.

The existence of a ``good'' subgroup is dataset dependent. Assumption~\ref{assm: r star} guarantees that a ``good'' subgroup of the data exists, and we should only expect to have performance guarantees in such a setting.

Finally, Theorem~\ref{thm: crs convergence} guarantees that Assumption~\ref{assm: crs} is a reasonable approximation given sufficient data.

Assumption~\ref{assm: core group} is a stronger assumption at face value, so we discuss its necessity separately in the following subsections.

\subsection{Necessity of Assumption~\ref{assm: core group}}
We provide a negative result showing the necessity of Assumption~\ref{assm: core group}, i.e., that the core group selection is contained in $R^*$. The problem arises because subsets of the best \emph{overall} region do not necessarily obtain the minimum EPE among subsets of a fixed shape. 

Consider the following counterexample. The features are supported on $\cB=[-1, 1]$ and consist of three regions.

In $R^* = [-1, 0]$, $T|X$ follows a Cox model with moderate signal: $T|X$ has hazard function $\lambda(t, x) = e^{\b_1 x}$ when $x \in R^*$.

In $R^+=[1-\e, 1]$, $T|X$ also follows a Cox model with a larger signal, but which does not compensate for the smaller region size: $\lambda(t, x) = e^{\b_2 x}$ when $x\in R^+$ with $|\b_2| > |\b_1|$.

Finally, in $R^\circ = \cB \setminus (R^* \cup R^+) =(0, 1-\e)$, we have that $T|X \sim \mathrm{Unif}(\{0, \infty\})$.

Figure~\ref{fig: list decodable} shows the EPE for all possible regions with a grid size of $0.1$. For this example, we used $\e = 0.1$, $\b_1 = -4$, and $\b_2 = 25$ with $n=4000$ points for the whole dataset with which to compute Monte Carlo estimates for the EPE of each region. As in Figure~\ref{fig: epe counterex}, the cell with bottom-left corner at $(a,b)$ corresponds to the EPE for the region $[a,b]$. The color of each cell denotes the EPE value, given by the color bar on the right of the plot.

Due to the combination of a favorable feature distribution (a larger region, making the units easier to distinguish) and a moderate signal conditional on the features, the minimum EPE is obtained by $R^*=[-1, 0]$, denoted by the cyan outlined square. The black outlined squares denote intervals of size $0.1$. If we consider core groups of this size, then the minimum among such groups will be obtained by $R^+=[0.9, 1]$, denoted by the red outlined square. As this is not a subset of $R^*$, even after the expansion step, we will not obtain the correct region.

\begin{figure}
    \centering
    \includegraphics[width=0.5\linewidth]{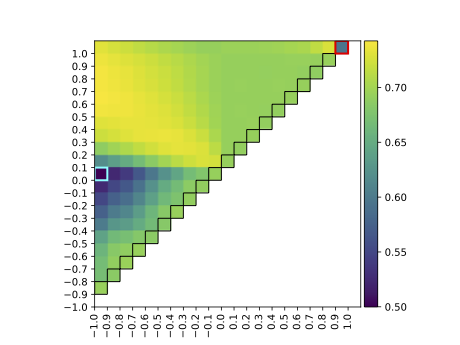}
    \caption{Counterexample showing the need for Assumption~\ref{assm: core group}. The cell with coordinates $(a, b)$ at its bottom-left corner denotes the region $[a, b]$. The subgroup $[-1, 0]$ minimizes the EPE and follows a Cox model (cyan boxed cell), but among the regions considered for the core group (black boxed cells), the interval $[0.9, 1]$ has the minimum EPE (red boxed cell).}
    \label{fig: list decodable}
\end{figure}

\subsection{Behavior of the Core Group Selection Procedure} \label{appendix: core group}
Our primary argument for the validity of Assumption~\ref{assm: core group} is via the empirical results on synthetic and real data. Nevertheless, for the sake of further understanding the behavior of the core group selection procedure, we also analyze a setting with ``lower level'' conditions on the data distribution where we can provide guarantees on the core group selection step. This is the content of Proposition~\ref{thm: sufficient for core}.

\begin{prop} \label{thm: sufficient for core}
Suppose that for all $t, x, x'$, we have 
\begin{equation} \label{eq: best pred}
    |\log \lambda(t, x) - \log \lambda(t, x')| \leq |\b^\T (x-x')|.
\end{equation}
Furthermore, suppose there exists $R^\circ \supseteq R^*$ and a constant $c < 1$ such that
\[
|\log \lambda(t, x) - \log \lambda(t, x')| \leq c|\b^\T (x-x')|
\]
whenever at least one of $x,x' \not\in R^\circ$.

Let $R_{\mathrm{cs}} \subseteq R^*$ be an axis-aligned box which is a subset of the ground truth region. (This box defines the core group shape, hence the subscript ``$\mathrm{cs}$'': core shape. For instance, $R_{\mathrm{cs}}$ may be the $\ell_\infty$ ball of radius $\e$ for some small $\e$.) Let $\cR = \{x_0 + R_{\mathrm{cs}} \: : \: x_0 \in \R^d, \; x_0 + R_{\mathrm{cs}} \subseteq B \}$ be the set of all translations of $R_{\mathrm{cs}}$ which are still contained in the support of the features, i.e., the set of possible core group regions. Then any minimizer of the EPE over $\cR$ must be a subset of $R^\circ$, i.e.,
\[ R \in \argmin_{R' \in \cR} \mathrm{EPE}(R') \quad \Longrightarrow \quad R \subseteq R^\circ. \]
\end{prop}

Before we give the proof, we note that the assumptions of Proposition~\ref{thm: sufficient for core} are incompatible with the assumptions for Theorem~\ref{thm: main}. In particular, it is easy to show that when $\lambda(t,x) = \lambda_0(t)e^{h(x)}$ and $|\lambda(t,x)-\lambda(t,x')|\leq|\b^\T(x-x')|$ for all $x,x'$, then $h(x)$ must be continuous in $x$. However, the assumptions of Theorem~\ref{thm: main} (specifically the combination of Assumption~\ref{assm: r star} and the requirement that $|h(x)-\b^\T x|\geq C_\e > 0$ for $x\not\in R^*$) imply a discontinuity in $h(x)$ at the boundary of $R^*$. Moreover, Proposition~\ref{thm: sufficient for core} only implies that the core group is a subset of $R^\circ$ rather than $R^*$. Thus, we emphasize that Proposition~\ref{thm: sufficient for core} is only meant to provide the reader with additional (theoretical) intuition for the behavior of the core group selection process.

\begin{proof}
First, note that the EPE for any core group $R = x_0 + R_{\mathrm{cs}} \subseteq R^*$ which is a subset of $R^*$ is equal to the following:
\begin{equation} \label{eq: min epe expression}
\mathrm{EPE}(R) = \E[H(\b^\T((X+x_0) - (X'+x_0)))] = \E[H(\b^\T(X-X'))] := \mathrm{EPE}^*.
\end{equation}
We can first show that this is indeed the minimum value of the EPE over $R \in \cR$. Let $R = x_0 + R_{\mathrm{cs}} \in \cR$. We have the following:
\begin{align} 
\mathrm{EPE}(R) &\geq \E\l[ H\l(\log\frac{\lambda(T^*, X+x_0)}{\lambda(T^*, X'+x_0)}\r) \r] \label{eq: epe lb} \\[10pt]
&\geq \E[ H(\b^\T((X+x_0) - (X'+x_0))) ] \label{eq: epe assumption used} \\[10pt]
&= \mathrm{EPE}^*. \nn
\end{align}
Inequality~\eqref{eq: epe assumption used} holds by combining \eqref{eq: best pred} with the fact that $H(z)$ is increasing as a function of $|z|$ by Lemma~\ref{thm: peaked ce}. Thus, core regions which are subsets of the ground truth minimize the expected EPE. It remains to show that if a core region is \emph{not} contained in $R^\circ$, then it must have a strictly larger EPE.

Suppose that $R \not\subseteq R^\circ$. Since $R$ and $R^\circ$ are both closed axis-aligned boxes, it is easy to see that in fact $R \setminus R^\circ$ must have positive Lebesgue measure. Choose $x+x_0 \in R \setminus R^\circ$. The set $S=\{x' \: : \: \b^\T (x-x') = 0\}$ has zero Lebesgue measure, so $(R \setminus R^\circ) \setminus S$ must also have positive Lebesgue measure and is therefore nonempty, so we can choose $x'+x_0 \in (R \setminus R^\circ) \setminus S$.

By Lemma~\ref{thm: peaked ce}, the entropy is (strictly) decreasing as a function of the absolute value of the logit. But note for $x+x_0$ and $x'+x_0$, and for any value of $t^*$, we have
\[ |\log \lambda(t^*, x+x_0) - \log \lambda(t^*, x'+x_0)| \leq c|\b^\T(x+x_0-(x'+x_0))| < |\b^\T(x-x')|, \]
where the final strict inequality holds because $\b^\T(x-x')\neq 0$ and $c<1$. Thus, in a neighborhood of $X=x$ and $X'=x'$, we have that
\[ H\l(\log\frac{\lambda(T^*, X+x_0)}{\lambda(T^*, X'+x_0)}\r) > H(\b^\T(X-X')). \]
In particular, with positive probability, the integrand in the lower bound \eqref{eq: epe lb} for the EPE of $R$ is strictly larger than the integrand in the expression \eqref{eq: min epe expression} for the minimum EPE, and it is bounded below by the integrand of $\mathrm{EPE}^*$ (though not necessarily strictly) by the logic of inequality~\eqref{eq: epe assumption used}. It follows that $\mathrm{EPE}(R) > \mathrm{EPE}^*$, completing the proof.
\end{proof}

\section{Experiment Details} \label{appendix: experiments}
In this section, we give additional details for reproducing the experiments, including complete algorithm implementations, definitions of additional metrics, dataset descriptions, and a brief discussion of compute resources.

\subsection{Algorithms} \label{appendix: algs expanded}
In this subsection we provide Python-esque pseudocode for the practical implementation of each algorithm. Each method takes the following arguments:
\begin{itemize}
    \item $X_{\textrm{adjust}}$: An $n\times d_1$ matrix containing the features to be used by the Cox model.
    \item $X_{\textrm{subgp}}$: An $n\times d_2$ matrix containing the features used to define the subgroup. Note that $X_{\textrm{adjust}}$ and $X_{\textrm{subgp}}$ can have the same features, partial overlap, or be completely disjoint. In our real data experiments, the features were generally disjoint. In the synthetic experiments, all features were used by both the Cox model and the subgroup definition.
    \item $Y$: A list of $n$ (event time, failure indicator) tuples, i.e., the $i$-th entry is $(t_i, \d_i)$.
    \item $B_{\textrm{subgp}}$: The bounding box for the subgroup features. In particular, all of the data have $x_{\mathrm{subgp}} \in B_{\textrm{subgp}}$, and the region $R$ returned by any method must have $R \subseteq B_{\textrm{subgp}}$.
\end{itemize}
Each method also accepts algorithm-specific hyperparameters, e.g., the core group size and rejection score quantile for DDGroup, or the maximum tree depth and minimum leaf size for the Cox tree. Given these inputs and valid hyperparameters, each method returns a region $R$ specifying the subgroup. The associated Cox model is always defined by fitting to all training data in $R$.

Some of the algorithms rely on subroutines for which we provide only a ``plain English'' description, rather than pseudocode. We refer readers to the associated code repository for the full implementations of all of the algorithms.

\subsubsection{Base}
\begin{python}
def base(X_adjust, X_subgp, Y, B_subgp):
    return B_subgp
\end{python}

\subsubsection{Random}
\begin{python}
def random(X_adjust, X_subgp, Y, B_subgp):
    n, d = X_subgp.shape
    
    selected_points = rng.sample(range(n), 2 * d) # Sample 2d points from X_adjust without replacement
    
    R_subgp = bounding_box(X_subgp[selected_points])
    
    return R
\end{python}

\subsubsection{Survival Tree}
\begin{python}
def survival_tree(X_adjust, X_subgp, Y, B_subgp, max_depth, min_samples_leaf):
    st = Tree(
        max_depth=max_depth,
        min_samples_leaf=min_samples_leaf,
        splitting_criterion=logrank_statistic
    )
    st.fit(X_adjust, X_subgp, Y) # Standard tree fitting procedure. Recursively choose splits of X_subgp features which minimizes the node impurity as measured by the logrank statistic.
    
    regions, betas = tree_to_bounding_boxes(st, depth, B) # Returns the bounding boxes and Cox models associated to each leaf of the tree.

    best_R, best_beta = argmin([EPE(X_adjust, X_subgp, R, beta) for R, beta in zip(regions, betas)]) # Compute the training EPE in each leaf and select the region with the lowest EPE.
    return best_R
\end{python}

\subsubsection{Cox Tree}
\begin{python}
def cox_tree(X_adjust, X_subgp, Y, B_subgp, max_depth, min_samples_leaf):
    ct = Tree(
        max_depth=max_depth,
        min_samples_leaf=min_samples_leaf,
        splitting_criterion=EPE
    )
    ct.fit(X_adjust, X_subgp, Y) # Standard tree fitting procedure. Recursively choose splits of X_subgp features which minimizes the node impurity as measured by the EPE.
    
    regions, betas = tree_to_bounding_boxes(ct, depth, B) # Returns the bounding boxes and Cox models associated to each leaf of the tree.

    best_R, best_beta = argmin([EPE(X_adjust, X_subgp, R, beta) for R, beta in zip(regions, betas)]) # Compute the training EPE in each leaf and select the region with the lowest EPE.
    return best_R
\end{python}

\subsubsection{PRIM}
Refer to \cite{friedman1999bump} for the original algorithm. PRIM relies on two subroutines: peeling and pasting. These can be understood as follows:
\paragraph{Peeling}
\begin{itemize}
\item Start with the full feature space (a large “box”).

\item Iteratively remove a small fraction $\a$ (e.g., 5\%) of the data from one face of the box. Each possible peel corresponds to shrinking the box along one variable by trimming off the lowest or highest values.

\item Choose the peel that maximally reduces the EPE within the remaining box.
\end{itemize}
This procedure is iterated until there are no peels which reduce the EPE, or until a minimum size is reached.

\paragraph{Pasting}
\begin{itemize}
\item Start from the bounding box obtained at the end of the peeling procedure.

\item Try \emph{adding} a small fraction $\a$ of the original data back to one face of the box. Each possible paste corresponds to expanding the box along one variable by including an $\a$ fraction of the data which is closest to that face of the box.

\item Choose the pasting operation which causes the greatest reduction in EPE. If no pasting operation reduces the EPE, terminate the algorithm and return the current region.
\end{itemize}

\begin{python}
def prim(X_adjust, X_subgp, Y, B_subgp, alpha, min_support_size):
    R = B_subgp.copy()
    current_metric_val = inf
    support = min_support_size * len(X_adjust) # min_support_size given as a fraction of the total dataset.

    # Peeling steps
    while len(in_region(X_subgp, R)) >= support:
        new_R, new_metric_val = peel(X_adjust, X_subgp, Y, R, alpha)

        if new_metric_val >= metric_val:
            break
        else:
            metric_val = new_metric_val
            R = new_R

    R = bounding_box(current_X_subgp)

    # Pasting steps
    while True:
        new_R, new_metric_val = paste(X_adjust, X_subgp, Y, R, alpha)
        if new_metric_val >= metric_val:
            break
        else:
            R = new_R
            metric_val = new_metric_val

    return R
\end{python}

\subsubsection{DDGroup}
Refer to \cite{izzo23subgroup} for the DDGroup algorithm in a linear regression context. We modify the algorithm by using the EPE for the core group selection step and the CRS for the rejection step (or the C-index/partial likelihood alternatives). We first give a high-level overview of the algorithm in Pythonic pseudocode.

\begin{python}
def ddgroup(X_adjust, X_subgp, Y, B_subgp, core_size, rejection_threshold, core_metric, rejection_metric):
    core_ind = core_group(X_adjust, X_subgp, Y, core_size, core_metric) # Computes core_metric on the core_size nearest neighbors of each point in the dataset. Returns the indices of the neighborhood with lowest value of core_metric.
    
    beta_hat = fit_cox(X_adjust[core_ind], Y[core_ind])
    
    scores = get_scores(X_adjust, Y, beta_hat, core_ind, rejection_metric) # Compute the rejection metric for each point in the dataset using the core group selected above.

    abs_threshold = quantile(scores, rejection_threshold)
    labels = scores < abs_threshold # Reject all points which are in the bottom rejection_threshold quantile of rejection scores.

    R = grow_region(X_subgp, labels, B_subgp, center=mean(X_subgp[core_ind], axis=0)) # Starting from the mean of the core group, expand the size of the bounding box until they collide with a rejected point.

    return R
\end{python}
DG is implemented with the EPE as the core metric and the CRS as the rejection metric. DG-CI is implemented with the C-index as the core metric and the C-index based rejection scores defined in Section~\ref{sec: algs}. DG-PL is implemented with the partial likelihood as the core metric and the partial likelihood-based rejections scores also defined in Section~\ref{sec: algs}. Finally, DG-NE just returns the bounding box of the core group.

\subsubsection{Why not include the change-plane Cox model?}
The change-plane Cox model, introduced by \cite{wei2018change}, nominally considers the same problem as our work: subgroup discovery with the Cox model. However, our problem setting differs from theirs in two critical ways.

The first is the manner in which the subgroups themselves are defined. \cite{wei2018change} defines two subgroups which are separated by a hyperplane. On the other hand, our work considers subgroups defined as axis-aligned boxes. This approach has precedent in previous works \citep{izzo23subgroup}. While both the hyperplane and box subgroups are reasonably interpretable, the box subgroups have the added advantage of being standard practice for specifying inclusion criteria for clinical trials, which is an important motivating use case for subgroup discovery. See \cite{friedman2015fundamentals} pg. 129 on defining patient strata (subgroups); this is precisely the axis-aligned box setting in our framework. Thus, from a purely algorithmic view, our setting and \cite{wei2018change} are incomparable---neither subsumes the other---but from a practical perspective, we believe the axis-aligned approach to be superior.

Another important difference between our work and \cite{wei2018change} is the modeling assumptions on the data. In particular, \cite{wei2018change} assumes a very strict relationship on the two subgroups: specifically, it is assumed that the hazards of the two discovered groups are proportional. In contrast, we make minimal assumptions outside of the region to be discovered other than that the Cox model does not fit it well. Due to the weaker assumptions, this should make our problem definition more challenging but also more practically relevant.

Regarding adapting \cite{wei2018change} method to our setting, we believe their proposed algorithm is fundamentally incompatible with our problem. There are two major issues. The first is that their algorithm critically hinges on constructing a ``sieve'' of potential change planes which define the possible subgroups. The vectors in this sieve consist of normalized eigenvectors of various conditional feature matrices; these eigenvectors will not be axis-aligned in general. Thus, a subgroup defined by the two sides of a plane defined by these vectors will not be an axis-aligned box, making it incompatible with our goal. Second, even if we ignore this difficulty, the way that the final subgroup selection is performed also relies critically on the strong modeling assumption that the two subgroups have hazards which are proportional to each other. As this assumption does not hold in our settings, the selection procedure becomes invalid. Given these many significant obstacles to employing the algorithm for our problem, we excluded it from the suite of algorithms tested.

\subsection{Additional Metric: Rejection Fraction} \label{sec: goodness of fit metrics}
Let $\{(x_i, y_i, \d_i)\}_{i=1}^m$ be a subgroup of points with corresponding Cox model $\b$. For each point $x_i$, we compute the CRS $\tau^*_i$ for $x_i$ with respect to the rest of the points in the subgroup, as defined in Section~\ref{sec: crs}. The \emph{rejection fraction at level $\a$} is then
\[
\frac1m\sum_{i=1}^m \I\{\tau^*_i < \a\},
\]
i.e., the fraction of points in the group whose CRS value is below $\a$. In the case of no censoring and a well-specified model (in which case the CRS value is actually equal to a tail probability), we would expect the rejection fraction to be roughly equal to $\a$. A significantly higher rejection fraction indicates that many of the survival time ranks are taking a more extreme value than what is predicted by the model, indicating a poor fit.

\subsection{Datasets} \label{appendix: datasets}

\paragraph{Table~\ref{tab: synth-counter}}
The data were generated according to the counterexample discussed in Section~\ref{sec: epe counterex}. Namely, the features are 1D and drawn uniformly from $[0, 1]$. The ground truth hazard function is
\[
\lambda(t, x) = \begin{cases}
e^{m x} & 0 \leq x < c \\
e^{m x - b} & c \leq x \leq 1
\end{cases}
\]
with $m=10$, $b=2$, and $c=0.4$. We generated $n=4000$ datapoints for the training data.

\paragraph{Table~\ref{tab: synth-nonlinear}}
This dataset used $n=4000$ datapoints, 2000 used for training and 2000 used for testing,
also in dimension $d=2$.
The ground truth region was $R^*=[-1/6^d, 1/6^d]^d$
and the ground truth Cox coefficients were $\b^*=10\cdot \mathbf{1}$, where $\mathbf{1}\in \R^d$ is the all-1s vector.
Given a feature vector $x\not\in R^*$, the nonlinear transformation we used to generate the data outside of $R^*$ was 
\[
\tilde{x}[1] = 10 \cdot \mathrm{sin}(100 \cdot x[1]^2), \quad \tilde{x}[2] = x[2],
\]
where $\tilde{x}[i], \, x[i]$ denote the $i$-th entries of the transformed features $\tilde{x}$ and the original features $x$, respectively. The survival times outside of $R^*$ were then generated from a Cox model with baseline hazard $\lambda_0(t) \equiv 1$ and Cox coefficients $\b^{\textrm{out}}=0.5\cdot \mathbf{1}$, but applied to the transformed features $\tilde{x}$.

\paragraph{METABRIC Study \citep{curtis2012metabric} (MBRIC)}
We used the version of this dataset from \cite{katzman2018deepsurv}, available at the associated \href{https://github.com/jaredleekatzman/DeepSurv/tree/master/experiments/data/metabric}{GitHub repository}.
The dataset has 1904 samples of which 1103 (57.9\%) are uncensored. We studied the MKI67 variable (MKI), which measures the expression of a gene which encodes the Ki-67 protein, a common cancer marker \citep{uxa2021metabricexplanation}.

The remaining real datasets are all provided by the sksurv Python package \citep{sksurv}.

\paragraph{AIDS Clinical Trial (AIDS)}
This dataset has 1151 samples and extremely high censoring; only 96 (8.3\%) of patients are uncensored. The endpoint is the onset of AIDS rather than death for this dataset. We studied three features:
\begin{itemize}
    \item CD4, a measurement of certain types of white blood cells;
    \item Karnofsky performance scale (Karnof), a measure introduced by \cite{karnofsky1948use} which measures a patient's ability to perform daily activities; and
    \item Prior ZDV use (prior), the number of months a patient has used the antiretroviral drug zidovudine.
\end{itemize}

\paragraph{German Breast Cancer Study Group 2 (GBSG2)}
This dataset has 686 samples of which 299 (43.6\%) are uncensored. We studied the tumor size feature (tsize).

\paragraph{Veterans' Administration Lung Cancer Trial (VLC)}
This dataset consists of 137 samples, of which 128 (93.4\%) are uncensored. We studied one feature, the Karnofsky score, which is a scale that measures a patient's ability to function during the progression of a disease \citep{karnofsky1948use}.

\paragraph{Worchester Heart Attack Study (WHAS)}
This dataset consists of 500 samples, of which 215 (43\%) are uncensored. We studied 3 distinct features: length of stay (los), systolic blood pressure (sysbp), and diastolic blood pressure (diasbp).

\subsection{Expanded Results} \label{appendix: extra results}
Here we report the results contained in Table~\ref{tab: real} including uncertainty as measured by standard error of the mean.

\begin{table}[h!]
\centering
\caption{Expanded real results.}
\resizebox{\textwidth}{!}{
\begin{tabular}{ll|ccccccccc}
\toprule
& & Base & Rand & PRIM & ST & CT & DG-PL & DG-CI & DG-NE & DG \\
\midrule
\rowcolor{gray!10}AIDS-CD4 & EPE & 0.54 (0.02) & 0.59 (0.13) & 0.53 (0.02) & 0.58 (0.10) & 0.50 (0.08) & 0.54 (0.09) & 0.59 (0.06) & 0.46 (0.09) & 0.43 (0.04) \\
& Rej@10\% & 0.07 (0.01) & 0.11 (0.05) & 0.06 (0.01) & 0.08 (0.03) & 0.12 (0.05) & 0.08 (0.03) & 0.10 (0.02) & 0.05 (0.02) & 0.07 (0.05) \\
\rowcolor{gray!10}&C-Index & 0.71 (0.01) & 0.72 (0.06) & 0.73 (0.01) & 0.71 (0.04) & 0.74 (0.04) & 0.72 (0.05) & 0.68 (0.03) & 0.75 (0.04) & 0.76 (0.03) \\
& Size & 1.00 (0.00) & 0.17 (0.03) & 0.87 (0.04) & 0.32 (0.11) & 0.17 (0.02) & 0.73 (0.14) & 0.55 (0.15) & 0.14 (0.01) & 0.20 (0.03) \\
\midrule
\rowcolor{gray!10}AIDS-Karnof & EPE & 0.62 (0.03) & 0.58 (0.11) & 0.62 (0.03) & 0.65 (0.07) & 0.72 (0.22) & 0.62 (0.04) & 0.72 (0.22) & 0.77 (0.22) & 0.38 (0.07) \\
& Rej@10\% & 0.09 (0.01) & 0.01 (0.01) & 0.15 (0.01) & 0.11 (0.03) & 0.19 (0.02) & 0.07 (0.01) & 0.06 (0.01) & 0.07 (0.02) & 0.01 (0.01) \\
\rowcolor{gray!10}&C-Index & 0.66 (0.03) & 0.68 (0.07) & 0.66 (0.03) & 0.64 (0.05) & 0.69 (0.07) & 0.66 (0.03) & 0.69 (0.07) & 0.67 (0.07) & 0.84 (0.05) \\
& Size & 1.00 (0.00) & 0.13 (0.01) & 0.96 (0.02) & 0.21 (0.09) & 0.44 (0.10) & 0.82 (0.12) & 0.65 (0.14) & 0.13 (0.01) & 0.15 (0.01) \\
\midrule
\rowcolor{gray!10}AIDS-prior & EPE & 0.70 (0.00) & 0.67 (0.04) & 0.70 (0.00) & 0.69 (0.05) & 0.68 (0.04) & 0.66 (0.03) & 0.70 (0.00) & 0.67 (0.05) & 0.65 (0.04) \\
& Rej@10\% & 0.14 (0.00) & 0.21 (0.05) & 0.06 (0.00) & 0.06 (0.00) & 0.06 (0.00) & 0.11 (0.03) & 0.14 (0.00) & 0.05 (0.00) & 0.13 (0.03) \\
\rowcolor{gray!10}&C-Index & 0.46 (0.02) & 0.55 (0.04) & 0.46 (0.02) & 0.56 (0.04) & 0.57 (0.04) & 0.52 (0.06) & 0.46 (0.02) & 0.57 (0.04) & 0.58 (0.06) \\
& Size & 1.00 (0.00) & 0.16 (0.02) & 0.94 (0.03) & 0.10 (0.01) & 0.17 (0.01) & 0.64 (0.15) & 1.00 (0.00) & 0.12 (0.01) & 0.16 (0.01) \\
\midrule
\rowcolor{gray!10}GBSG2-tsize & EPE & 0.68 (0.00) & 0.62 (0.04) & 0.68 (0.00) & 0.68 (0.02) & 0.63 (0.04) & 0.68 (0.00) & 0.65 (0.03) & 0.61 (0.04) & 0.61 (0.04) \\
& Rej@10\% & 0.14 (0.00) & 0.13 (0.03) & 0.08 (0.01) & 0.04 (0.01) & 0.05 (0.02) & 0.13 (0.02) & 0.11 (0.02) & 0.05 (0.02) & 0.07 (0.02) \\
\rowcolor{gray!10}&C-Index & 0.57 (0.01) & 0.62 (0.04) & 0.55 (0.01) & 0.57 (0.03) & 0.62 (0.04) & 0.58 (0.01) & 0.64 (0.05) & 0.64 (0.04) & 0.64 (0.04) \\
& Size & 1.00 (0.00) & 0.15 (0.02) & 0.94 (0.01) & 0.26 (0.07) & 0.11 (0.01) & 0.90 (0.10) & 0.39 (0.14) & 0.12 (0.01) & 0.12 (0.01) \\
\midrule
\rowcolor{gray!10}MBRIC-MKI & EPE & 0.69 (0.00) & 0.70 (0.01) & 0.69 (0.00) & 0.70 (0.01) & 0.72 (0.01) & 0.69 (0.00) & 0.70 (0.01) & 0.72 (0.01) & 0.69 (0.01) \\
& Rej@10\% & 0.08 (0.00) & 0.12 (0.01) & 0.08 (0.00) & 0.14 (0.01) & 0.13 (0.01) & 0.08 (0.00) & 0.09 (0.01) & 0.14 (0.01) & 0.11 (0.01) \\
\rowcolor{gray!10}&C-Index & 0.49 (0.01) & 0.50 (0.02) & 0.49 (0.01) & 0.52 (0.02) & 0.49 (0.02) & 0.49 (0.01) & 0.49 (0.02) & 0.49 (0.03) & 0.54 (0.01) \\
& Size & 1.00 (0.00) & 0.17 (0.02) & 0.84 (0.05) & 0.26 (0.04) & 0.25 (0.03) & 1.00 (0.00) & 0.64 (0.15) & 0.11 (0.01) & 0.38 (0.11) \\
\midrule
\rowcolor{gray!10}VLC-Karnof & EPE & 0.57 (0.02) & 0.42 (0.14) & 0.58 (0.02) & 0.45 (0.10) & 0.22 (0.05) & 0.32 (0.02) & 0.34 (0.06) & 0.19 (0.05) & 0.33 (0.06) \\
& Rej@10\% & 0.04 (0.01) & 0.12 (0.04) & 0.03 (0.01) & 0.07 (0.03) & 0.21 (0.04) & 0.11 (0.05) & 0.10 (0.03) & 0.20 (0.04) & 0.09 (0.03) \\
\rowcolor{gray!10}&C-Index & 0.69 (0.02) & 0.84 (0.06) & 0.68 (0.02) & 0.77 (0.07) & 0.93 (0.02) & 0.92 (0.03) & 0.87 (0.03) & 0.93 (0.02) & 0.87 (0.04) \\
& Size & 1.00 (0.00) & 0.17 (0.02) & 0.88 (0.05) & 0.23 (0.05) & 0.20 (0.02) & 0.23 (0.03) & 0.22 (0.02) & 0.18 (0.02) & 0.28 (0.08) \\
\midrule
\rowcolor{gray!10}WHAS-DBP & EPE & 0.67 (0.01) & 0.83 (0.07) & 0.67 (0.01) & 0.65 (0.06) & 0.64 (0.07) & 0.67 (0.03) & 0.55 (0.08) & 0.74 (0.09) & 0.62 (0.07) \\
& Rej@10\% & 0.07 (0.01) & 0.01 (0.01) & 0.13 (0.00) & 0.21 (0.03) & 0.16 (0.03) & 0.04 (0.01) & 0.02 (0.01) & 0.20 (0.03) & 0.01 (0.01) \\
\rowcolor{gray!10}&C-Index & 0.61 (0.01) & 0.51 (0.06) & 0.61 (0.01) & 0.68 (0.04) & 0.66 (0.06) & 0.64 (0.03) & 0.75 (0.07) & 0.63 (0.05) & 0.70 (0.06) \\
& Size & 1.00 (0.00) & 0.14 (0.01) & 0.98 (0.01) & 0.15 (0.02) & 0.16 (0.01) & 0.62 (0.15) & 0.20 (0.09) & 0.15 (0.02) & 0.12 (0.01) \\
\bottomrule
\end{tabular}
}
\end{table}

\subsection{Compute Resources}
Our experiments consisted of a large number of lightweight individual jobs--one per combination of (dataset, train/test split, subgroup discovery method, hyperparameter setting). We used an internal SLURM-managed compute cluster and only CPU hardware for these jobs, specifically, AMD EPYC 7542 and 7543 32-Core Processors, Intel(R) Xeon(R) Silver 4108 CPU @ 1.80GHz, and similar. No GPUs were required for these experiments.


\end{document}